\documentclass[twoside]{article}
\usepackage[left=2cm,right=2cm,top=2cm,bottom=2cm]{geometry}
\usepackage{multicol}
\usepackage{amsmath}
\usepackage{amsfonts}
\usepackage{amssymb}
\usepackage{graphicx}
\usepackage{color}
\usepackage[normalem]{ulem}
\usepackage[utf8]{inputenc}
\usepackage{amsfonts}
\usepackage{color}
\usepackage[normalem]{ulem}
\usepackage{nomencl}
\makenomenclature

\newenvironment{proof}{\noindent{\sc Proof.}}{\qed}
\newtheorem{theorem}{Theorem}[section]
\newtheorem{lemma}{Lemma}[section]
\newtheorem{cor}{Corollary}[section]
\newtheorem{rem}{Remark}[section]
\newtheorem{prop}{Proposition}[section]
\newtheorem{uda}{Example}[section]
\newcommand{\qed}{$\blacksquare$}
\def\hindu{\arabic}

\renewcommand{\theequation}{\hindu{section}.\hindu{equation}}
\def\bhag#1{\noindent
\setcounter{equation}{0}
\section{#1}
}

\def\RR{{\mathbb R}}
\def\CC{{\mathbb C}}
\def\ZZ{{\mathbb Z}}

\def\TT{\mathbb T}

\def\bs#1{{\boldsymbol{#1}}}
\def\x{\mathbf{x}}
\def\k{\mathbf{k}}
\def\y{\mathbf{y}}
\def\u{\mathbf{u}}

\def\z{\mathbf{z}}

\def\v{\mathbf{v}}

\def\j{\mathbf{j}}

\def\O{{\cal O}}

\def\E{{\cal E}}

\def\YY{\mathbb{Y}}

\def\be{\begin{equation}}
\def\ee{\end{equation}}
\def\bea{\begin{eqnarray}}
\def\eea{\end{eqnarray}}
\def\eref#1{(\ref{#1})}
\def\disp{\displaystyle}

\def\binom#1#2{\small{\left(\!\!\begin{array}{c}{#1}\\{#2}
\end{array}\!\!\right)}}
\def\donchitre#1#2{\vskip 6.5cm\noindent
\parbox[t]{1in}{\special{eps:#1.eps x=6.5cm y=5.5cm}}
\hbox to 7cm{}\parbox[t]{0.0cm}{\special{eps:#2.eps x=6.5cm y=5.5cm}}}

\def\tn{|\!|\!|}
\def\XX{{\mathbb X}}
\def\BB{{\mathbb B}}

\def\bs#1{{\boldsymbol{#1}}}

\def\Lip{\mbox{{\rm Lip}}}
\title{A direct approach for function approximation on data defined manifolds}

\author{
 H.~N.~Mhaskar\thanks{
Institute of Mathematical Sciences, Claremont Graduate University, Claremont, CA 91711. 
\textsf{email:} hrushikesh.mhaskar@cgu.edu.
The research is supported in part by NSF grant DMS 2012355.}
 }
 \date{}
\begin{document}
\makenomenclature
\maketitle
\begin{abstract}
In much of the literature on function approximation by deep networks, the function is assumed to be defined on some known domain, such as a cube or a sphere. 
In practice, the data might not be dense on these domains, and therefore, the approximation theory results are observed to be too conservative.
In manifold learning, one assumes instead that the data is sampled from an unknown manifold; i.e., the manifold is defined by the data itself.
Function approximation on this unknown manifold is then a two stage procedure: first, one approximates the Laplace-Beltrami operator (and its eigen-decomposition) on this manifold using a graph Laplacian, and next, approximates the target function using the eigen-functions. 
Alternatively, one estimates first some atlas on the manifold and then uses local approximation techniques based on the local coordinate charts.

In this paper, we propose a more direct approach to function approximation on \emph{unknown}, data defined manifolds without computing the eigen-decomposition of some operator or an atlas for the manifold, and without any kind of training in the classical sense.  
Our constructions  are universal; i.e., do not require the knowledge of any prior on the target function other than continuity on the manifold.
We estimate the degree of approximation.
For smooth functions, the estimates do not suffer from the so-called saturation phenomenon.
We demonstrate via a property called good propagation of errors how the results can be lifted for function approximation using deep networks where each channel evaluates a Gaussian network on a possibly unknown manifold.
\end{abstract}

\bhag{Introduction}\label{intsect}
One of the main problems of machine learning is the following. 
Given data $\{(\y_j,f(\y_j)+\epsilon_j)\}_{j=1}^M$,   where $f$ is an unknown function, $\y_j$'s are sampled randomly from a probability distribution $\mu^*$ defined on a subset of $\RR^Q$ for some typically high dimension $Q$, and $\epsilon_j$'s are realizations of a mean zero random variable, find an approximation $P$ from a class $V_n$ to $f$ \cite{girosi1990networks, CucSma02, zhoubk_learning}, where $\{V_n\}$ is a nested sequence of subsets of $L^2(\mu^*)$. 
In practice, this approximation is found by empirical risk minimization, assuming some prior on $f$, such as that it belongs to some reproducing kernel Hilbert space with a known kernel, or that it has a certain number of derivatives, or  that it satisfies some conditions on its Fourier transform. 
To set up the minimization problem, one needs to know in advance the complexity of the model $P$, typically, the number of parameters desired to be estimated.
In theory, the usual way of estimating this number is to estimate the so called approximation error, $\inf_{P\in V_n}\mathbb{E}_{\mu^*}((f-P)^2)$.
Necessarily, this results in a fundamental gap in the theory, namely, that the minimizer of the empirical risk may have no connection with the minimizer of the approximation error.

Since the fundamental problem is one of function approximation, it is natural to wonder if appropriate tools in approximation theory can be developed in order to close this gap.
One of the difficulties in doing so is that most of the results in classical approximation theory assume that the approximation takes place on a known domain, such as the cube, or Euclidean space, or sphere or similar known manifold. 
In turn, this requires that the data should be dense on this domain; i.e., the domain should be the (exact) support of $\mu^*$.
The problem is that $\mu^*$ being unknown, it is not possible to ensure this requirement.

During this century, manifold learning has sought to ameliorate the situation, with many practical applications. 
An early introduction to this topic is in the special issue \cite{achaspissue} of Applied and Computational Harmonic Analysis, edited by Chui and Donoho.
In this theory, one assumes that the support of $\mu^*$ is an unknown smooth compact  connected manifold; for simplicity, even that $\mu^*$ is the Riemannian volume measure for the manifold, normalized to be a probability measure.
Following, e.g., \cite{belkin2003laplacian, belkinfound, niyogi2, lafon, singer}, one constructs first a ``graph Laplacian'' from the data, and finds its eigen-decomposition. 
It is proved in the above mentioned papers that as the size of the data tends to infinity, the graph Laplacian converges to the Laplace-Beltrami operator on the manifold and the eigen-values (respectively, eigen-vectors) converge to the corresponding quantities on the manifold.
A great deal of work is devoted to studying the geometry of this unknown manifold (e.g., \cite{jones2010universal, liao2016adaptive}), based on the so called heat kernel.
The theory of function approximation on such manifolds is also well developed (e.g., \cite{mauropap, eignet, heatkernframe, compbio, modlpmz}).

All this work depends upon a two stage procedure - finding the eigen-decomposition of the graph Laplacian and then using approximation in terms of the eigen-vectors/eigen-functions. 
Once more, this leads to errors not just from the approximation of the target function but also from the approximation of the eigen-decomposition of the Laplace-Beltrami operator itself.
In recent years, there are some efforts to  explore alternative approaches using deep networks (e.g., \cite{coifman_deep_learn_2015bigeometric, chui_deep, relu_manifold_chen2019,schmidt2019deep}). 
These papers also take a two-step approach: developing an atlas on the manifold first, and then using some local approximation schemes based on the local coordinate charts.

Our objective in this paper is to develop a single-shot method to solve the problem, knowing only the dimension of the manifold. 
In particular, we aim not to find any eigen-decomposition nor to learn any atlas on the manifold, but to give a direct construction that starts with the data and constructs an approximation without involving any optimization/training and with guaranteed approximation error estimated in a probabilistic sense.
Our approximation can be implemented as a Gaussian network; i.e., a function of the form $\x\mapsto \sum_k a_k\exp(-\lambda|\x-\y_k|_{2,Q}^2)$, where $|\cdot|_{2,Q}$ denotes the $\ell^2$ norm on $\RR^Q$.
The size of the data set required depends only on the dimension of the manifold and the smoothness of the target function measured in a technical manner as explained in this paper.
We will extend our results to approximation by deep Gaussian networks.

\bhag{Technical introduction and outline}\label{bhag:technicalint}

In this section, let us assume that the data $\y_j$ is sampled from some unknown manifold, uniformly with respect to the Riemannian volume element of that manifold.
One of the fundamental results in manifold learning is the following theorem of Belkin and Niyogi \cite{belkinfound}.
\begin{theorem}\label{theo:belkin_niyogi}
Let $\XX$ be a smooth, compact, $q$-dimensional sub-manifold of $\RR^Q$, $\mu^*$ be its Riemannian volume measure, normalized by $\mu^*(\XX)=1$, and $\Delta$ denote the Laplace-Beltrami operator on $\XX$. Then for a smooth function $f : \XX\to\RR$,
\be\label{belkinbd}
\lim_{t\to 0}\frac{1}{t(4\pi t)^{q/2}}\int_\XX \exp\left(-\frac{|\x-\y|_{2,Q}^2}{t}\right)(f(\y)-f(\x))d\mu^*(\y) = \Delta(f)(\x)
\ee 
uniformly for $\x\in\XX$, where $|\cdot|_{2,Q}$ denotes the $\ell^2$ norm on $\RR^Q$.
Equivalently, uniformly for $\x\in\XX$, we have
\be\label{belkin_saturation}
\left|\frac{1}{(4\pi t)^{q/2}}\int_\XX \exp\left(-\frac{|\x-\y|_{2,Q}^2}{t}\right)(f(\y)-f(\x))d\mu^*(\y)-t\Delta(f)(\x)\right|=o(t)
\ee
as $t\to 0+$.
\end{theorem}

From an approximation theory point of view, the theorem is more of a saturation theorem for approximating $f$ on $\XX$, analogous to the Voronowskaja estimates for Bernstein polynomials (\cite[Section~1.6.1]{lorentz2013bernstein}, See Appendix~\ref{bhag:saturation}). 
Thus, \eref{belkin_saturation} states that the rate of approximation of $f$ cannot be better than $\O(t)$, even if $f$ is infinitely differentiable, unless $f$ is in the null space of the Laplace-Beltrami operator.
This is to be expected because the Gaussian kernel involved is a positive operator.
In particular, this phenomenon holds even if $\XX$ is a Euclidean space rather than a manifold.
Moreover, the curvature of the manifold contributes to the saturation as well.
The Gaussian kernel has many advantages, invariance under translations and rotations is one of the them.
This plays a major role in the proof of Theorem~\ref{theo:belkin_niyogi}.
Nevertheless, it is natural to ask whether another kernel can be found that leads directly to the approximation of the target function $f$ on the manifold from the data without knowing the  manifold itself and without having to go through an expensive eigen-decomposition.
The curvature of the manifold will still affect the rate of convergence, but when applied to an affine space rather than a manifold, such a construction should lead to approximation without any saturation, without knowing what the affine space is (Remark~\ref{rem:saturation}).

The main objective of this paper is to demonstrate such a construction using certain localized kernels based on Hermite polynomials (Theorem~\ref{theo:manifoldprob}). 
This theorem gives an analogue of Theorem~\ref{theo:belkin_niyogi} to obtain function approximation on an unknown manifold based only on noise-corrupted samples on the manifold, and give estimates on the degree of approximation.
In the case when the approximation is done on an affine space rather than a manifold, our construction is free of any saturation, and does not need to know what the affine space is (Theorem~\ref{theo:plane}).

To recapture the advantage of the Gaussian kernel, we will study approximation by Gaussian networks. 
A (shallow) Gaussian network with $n$ neurons has the form $\x\mapsto\sum_{k=1}^n a_k\exp(-\lambda|\x-\y_k|_{2,Q}^2)$. A deep Gaussian network is constructed following a DAG structure, where each node (referred to as ``channel'' in the literature on deep learning) evaluates a Gaussian network.
Using  the close connection between Hermite polynomials and Gaussian networks (cf. \cite{mhasbk, convtheo, chuigaussian}), we can translate the  result about approximation on the manifold into a result on approximation by shallow Gaussian networks, where the input is assumed to lie on an unknown low dimensional manifold of the nominally high dimensional ambient space (Theorem~\ref{theo:shallow_net}).
In turn, using a property called ``good propagation of errors'' (Theorem~\ref{theo:good_propogation}), we will ``lift'' this theorem to estimate the degree of approximation by deep Gaussian networks, where each channel evaluates a Gaussian network on a similarly manifold-based data (Theorem~\ref{theo:deep_gaussian}).
The networks themselves are constructed from certain \textbf{pre-fabricated networks} in the ambient space to approximate the Hermite functions with a correspondingly high number of neurons.
However, we will give an explicit formula for such networks (Proposition~\ref{prop:poly_to_gaussian}), so that there is \textbf{no training required here}. The amount of information used in the final synthesis of the network will depend only on the dimension of the manifold on which the input lives.
We consider this to be a step in bringing approximation theory of deep networks closer to the practice, so that the results are proved in the setting of approximation on unknown manifolds analogous to diffusion geometry rather than on known domains.

The statement of the main results in this paper mentioned above require a good deal of background information on the theory of weighted polynomial approximation, which we defer to  Section~\ref{bhag:backwtpoly}.
We will state the main results about approximation on a manifold in Section~\ref{bhag:manifoldapprox}, and illustrate them using a simple numerical example in Section~\ref{bhag:numerical}. We explain our ideas about shallow and deep networks in Section~\ref{bhag:gaussnet}.
To develop the details required in the constructions and proofs, we start by summarizing the relevant facts from the theory of weighted polynomial approximation in Section~\ref{bhag:backwtpoly}.
Of particular interest is the approximation of a weighted polynomial using pre-fabricated Gaussian networks whose weights and centers do not depend upon the polynomial, as described in Section~\ref{bhag:hermite_to_gaussian}. 
Our main theorem  in the context of approximation on unknown affine spaces is stated and proved in Section~\ref{bhag:affine}. The proofs of the results in Section~\ref{bhag:manifoldapprox} and \ref{bhag:gaussnet} are given in Sections~\ref{bhag:manifoldproofs} and \ref{bhag:deepproofs} respectively.

\bhag{Approximation on manifolds}\label{bhag:manifoldapprox}
In this section, we state our main results on approximation on manifolds. 
The details and motivations for these constructions will be clearer after reading Sections~\ref{bhag:backwtpoly} and \ref{bhag:affine}. 
The notation on the manifolds is described in Section~\ref{bhag:manifoldback}, the results themselves are discussed in Section~\ref{bhag:manifoldmain}.

\subsection{Definitions}\label{bhag:manifoldback}
Let $Q\ge q\ge 1$ be integers, $\XX$ be a $q$ dimensional, compact, connected,  sub-manifold of $\RR^Q$ (without boundary), with geodesic distance $\rho$ and volume measure $\mu^*$, normalized so that $\mu^*(\XX)=1$.
We will identify the tangent space at $\x\in\XX$ with an affine space $\TT_\x(\XX)$ in $\RR^Q$ passing through $\x$. 
For any $\x\in\XX$, we need to consider in this section three kinds of balls. 
\be\label{balldef}
B_Q(\x,r):=\{\y\in \RR^Q : |\x-\y|_{2,Q} \le r\},\
B_\TT(\x,r):= \TT_\x(\XX)\cap B_Q(\x,r),\ 
\BB(\x,r):=\{\y\in \XX: \rho(\x,\y)\le r\}.
\ee

With this convention,
the exponential map $\E_\x$ at $\x\in\XX$ (based on the definition in \cite[Proposition~2.9]{docarmo_riemannian}) is a diffeomorphism of an open ball centered at $\x$ in  $\TT_\x(\XX)$ onto its image in $\XX$ such that  $\rho(\x,\E_\x(\u))=|\u-\x|_{2,Q}$. 
Since $\XX$ is compact, there exists $\iota^*>0$ such that for \emph{\textbf{every}} $\x\in\XX$, $\mathcal{E}_\x$ is defined on $B_\TT(\x,\iota^*)$, and $\rho(\x,\E_\x(\u))=|\u-\x|_{2,Q}$ for all $\u\in  B_\TT(\x,\iota^*)$.

We now define the smoothness class $W_\gamma(\XX)$.
If $f, g:\XX\to\RR$, the function $fg :\XX\to\RR$ is defined as usual by $(fg)(\x)=f(\x)g(\x)$ for $\x\in\XX$.
The space $C(\XX)$ is the space of all continuous real-valued functions on $\XX$, equipped with the supremum norm $\|\circ\|_\XX$.
The space $C^\infty(\XX)$ is the subspace of $C(\XX)$ comprising all infinitely differentiable functions on $\XX$. 
 Let $f\in C(\XX)$, $\gamma>0$. 
 We say that $f\in W_\gamma(\XX)$ if for every $\x\in\XX$, and $\phi\in C^\infty(\XX)$,  supported on $ \mathbb{B}(\x,\iota^*/2)$, the function $F_{\x,\phi}: \TT_\x(\XX) \to \RR$ defined by $F_{\x,\phi}(\u):=f(\mathcal{E}_\x(\u))\phi(\mathcal{E}_\x(\u))$ is in $W_\gamma(\TT_\x)$ in the sense described in Section~\ref{bhag:affine} (See \eqref{sobolevnormdef}, \eqref{planedegapproxdef}, and \eqref{affinesobnorm}).
 We define
 \be\label{manifold_sob_norm}
 \|f\|_{W_\gamma(\XX)}:=\sup_{\x\in\XX, \|\phi\|_{\XX}\le 1}\|F_{\x,\phi}\|_{W_\gamma(\TT_\x(\XX))}.
 \ee
If $\gamma$ is an integer and $f$ is $\gamma$ times differentiable on $\XX$ then $f\in W_\gamma(\XX)$.
The space $W_\gamma(\XX)$ can contain functions which are not differentiable. 
For example, we say that $f\in \Lip(\XX)$ if 
$$
\|f\|_{\Lip(\XX)}:=\sup_{\x,\y\in\XX, \x\not=\y}\frac{|f(\x)-f(\y)|}{\rho(\x,\y)} <\infty.
$$
We have $\Lip(\XX) \subset W_1(\XX)$.

Next, we define the approximation operators. 
The orthonormalized Hermite polynomial $h_k$ of degree $k$ is defined recursively by
\bea\label{recurrence}
h_k(x)&:=&\sqrt{\frac{2}{k}}xh_{k-1}(x)-\sqrt{\frac{k-1}{k}}h_{k-2}(x), \qquad k=2,3,\cdots,
\nonumber\\
&&h_0(x):=\pi^{-1/4},\ h_1(t):=\sqrt{2}\pi^{-1/4}x.
\eea
We write 
\be\label{uni_psi_def}
\psi_k(t):=h_k(t)\exp(-t^2/2), \qquad t\in\RR,\ k\in\ZZ_+.
\ee 
The functions $\{\psi_k\}_{k=0}^\infty$ are an orthonormal set with respect to the Lebesgue measure (cf. \eqref{uniortho}).
In the sequel, we fix an infinitely differentiable function $H :[0,\infty)\to [0,1]$, such that $H(t)=1$ if $0\le t\le 1/2$, and $H(t)=0$ if $t\ge 1$.
We define for $x\in\RR$, $m\in\ZZ_+$:
\be\label{fastproj}
\mathcal{P}_{m,q}(x):=\begin{cases}
\disp\pi^{-1/4} (-1)^m\frac{\sqrt{(2m)!}}{2^m m!}\psi_{2m}(x), &\mbox{ if $q=1$,}\\[2ex]
\disp \frac{1}{\pi^{(2q-1)/4}\Gamma((q-1)/2)}\sum_{\ell=0}^m (-1)^\ell\frac{\Gamma((q-1)/2+m-\ell)}{(m-\ell)!}  \frac{\sqrt{(2\ell)!}}{2^\ell \ell!}\psi_{2\ell}(x), &\mbox{ if $q\ge 2$,}
\end{cases}
\ee
and the kernel $\widetilde{\Phi}_{n,q}$ for $x\in\RR$, $n\in\ZZ_+$ by
\be\label{fastkerndef}
\widetilde{\Phi}_{n,q}(x):=\sum_{m=0}^{\lfloor n^2/2\rfloor}H\left(\frac{\sqrt{2m}}{n}\right)\mathcal{P}_{m,q}(x).
\ee
\noindent\textbf{Constant convention:}\\
\textit{
In the sequel, $c, c_1,\cdots$ will denote generic positive constants depending upon the dimension and other fixed quantities in the discussion, such as the norm. 
Their values
may be different at different occurrences, even within a single formula. The notation $A\sim B$ means $c_1A\le B\le c_2B$.
\qed}\\

\subsection{Approximation theorems}\label{bhag:manifoldmain}
The traditional machine learning paradigm is to consider data of the form $\{(\y_j, f(\y_j)+\epsilon_j)\}$, where $\y_j$'s are drawn randomly with respect to $\mu^*$ and $\epsilon_j$'s are random, mean $0$ samples from an unknown distribution. More generally, we assume here a noisy data of the form $(\y,\epsilon)$, with a joint probability distribution $\tau$ and assume further that  the marginal distribution of $\y$ with respect to $\tau$ has the form $d\nu^*=f_0d\mu^*$ for some $f_0\in C(\XX)$. 
In place of $f(\y)$, we consider a noisy variant $\mathcal{F}(\y,\epsilon)$, and denote
\be\label{Fdef} 
f(\y):=\mathbb{E}_\tau(\mathcal{F}(\y,\epsilon)|\y).
\ee
\begin{rem}\label{rem:offmanifold}
{\rm
In practice, the data may not lie on a manifold, but it is reasonable to assume that it lies on a tubular neighborhood of the manifold. 
Our notation accommodates this - if $\z$ is a point in a neighborhood of $\XX$, we may view it as a perturbation  of a point $\y\in\XX$, so that the noisy value of the target function is $\mathcal{F}(\y,\epsilon)$, where $\epsilon$ encapsulate the noise in both the $\y$ variable and the value of the target function.
An example is given in Example~\ref{uda:difficult_noise}.
\qed}
\end{rem}
Our approximation process is simple: given by
\be\label{festimator}
\widehat{F}_{n,\alpha}(Y;\x):=\frac{n^{q(1-\alpha)}}{M}\sum_{j=1}^M \mathcal{F}(\y_j,\epsilon_j) \tilde\Phi_{n,q}(n^{1-\alpha}|\x-\y_j|_{2,Q}), \qquad \x\in\RR^Q,
\ee
where $0<\alpha\le 1$.

Our main theorem is the following.
\begin{theorem}\label{theo:manifoldprob}
Let $\gamma>0$, 
$\tau$ be a probability distribution on $\XX\times \Omega$ for some sample space $\Omega$ such  the marginal distribution of $\tau$ restricted to $\XX$ is absolutely continuous with respect to $\mu^*$ with density $f_0\in W_\gamma(\XX)$.
We assume that
\be\label{ballmeasure}
\sup_{\x\in\XX, r>0}\frac{\mu^*(\BB(\x,r))}{r^q} \le c.
\ee
Let $\mathcal{F} : \XX\times \Omega\to \RR$ be a bounded function,  $f$  defined by \eref{Fdef} be in $W_\gamma(\XX)$, the probability density $f_0\in W_\gamma(\XX)$. 
Let $M\ge 1$, $Y=\{(\y_1,\epsilon_1),\cdots,(y_M,\epsilon_M)\}$ be a set of random samples chosen i.i.d. from $\tau$.
If
\be\label{alphacond}
0<\alpha<\frac{4}{2+\gamma}, \qquad \alpha\le 1,
\ee
then  for every $n\ge 1$, $0<\delta<1$ and $M\ge n^{q(2-\alpha)+2\alpha\gamma}\sqrt{\log (n/\delta)}$, we have with $\tau$-probability $\ge 1-\delta$:
\be\label{locprobest}
\left\|\widehat{F}_{n,\alpha}(Y;\circ)-ff_0\right\|_\XX\le c_1\frac{\sqrt{\|f_0\|_\XX}\|\mathcal{F}\|_{\XX\times \Omega}+\|ff_0\|_{W_\gamma(\XX)}}{n^{\alpha\gamma}}.
\ee
\end{theorem}

We record two corollaries of  Theorem~\ref{theo:manifoldprob} as separate theorems.
The first is the approximation of $f$ itself, assuming that $f_0\equiv 1$.
\begin{theorem}\label{theo:manifold_approx_prob}
With the set-up as in Theorem~\ref{theo:manifoldprob}, let
 $f_0\equiv 1$ (i.e., the marginal distribution of $\y$ with respect to $\tau$ is $\mu^*$). Then we have with $\tau$-probability $\ge 1-\delta$:
\be\label{locprobest_approx}
\left\|\widehat{F}_{n,\alpha}(Y;\circ)-f\right\|_\XX\le c_1\frac{\|\mathcal{F}\|_{\XX\times \Omega}+\|f\|_{W_\gamma(\XX)}}{n^{\alpha\gamma}}.
\ee
\end{theorem}

The second is a consequence analogous to Theorem~\ref{theo:belkin_niyogi}.

\begin{theorem}\label{theo:belkin_niyogi_analogue}
With the set-up as in Theorem~\ref{theo:manifoldprob}, we have with $\tau$-probability $\ge 1-\delta$:
\be\label{belkin_niyogi_analogue_bd}
\left\|\frac{n^{q(1-\alpha)}}{M}\sum_{j=1}^M \left(\mathcal{F}(\y_j,\epsilon_j)-f(\circ)\right) \tilde\Phi_{n,q}(n^{1-\alpha}|\circ-\y_j|_{2,Q}) \right\|_\XX\le c_1\frac{\sqrt{\|f_0\|_\XX}\|\mathcal{F}\|_{\XX\times \Omega}+\|ff_0\|_{W_\gamma(\XX)}}{n^{\alpha\gamma}}.
\ee
\end{theorem}
\begin{rem}\label{rem:belkin_comparison}
{\rm
To compare the estimate \eqref{belkin_niyogi_analogue_bd} with \eqref{belkin_saturation}, which is applicable with $\gamma=2$, we are tempted to take any $\alpha\in (0,1)$, set $t=n^{-2(1-\alpha)}$, and obtain the upper bound $t^A$ with $A=\alpha/(1-\alpha)$. 
Clearly, this bound tends to $0$ arbitrarily fast with $t$. 
However, the estimate \eqref{belkin_saturation} uses a fixed kernel, while the estimate \eqref{belkin_niyogi_analogue_bd} uses a kernel depending upon $t$. 
\qed}
\end{rem}

\begin{rem}\label{rem:saturation}
{\rm
Although the curvature of the manifold forces us to put limitations on  the rate of convergence in \eref{locprobest_approx}, this is not a saturation phenomenon. 
Thus, it is not ruled out that the rate can be much better than that given in \eref{locprobest_approx} for non-trivial functions.
\qed}
\end{rem}

\begin{rem}\label{rem:alphachoice}
{\rm
If $\gamma<2$, we may choose   $\alpha=1$ without knowing the value of $\gamma$.
The formula \eref{festimator} itself does not require any prior knowledge of the smoothness of $f$. 
\qed}
\end{rem}

\vskip-0.75cm
\bhag{Numerical example}\label{bhag:numerical}

We illustrate the theory using the following simple example. We let $\XX\subset \RR^3$ to be the helix defined by
\be\label{helixdef}
\x(t):=(\cos(\pi t), \sin(\pi t), \pi t), \qquad 0\le t\le 2\pi.
\ee
This does not satisfy the conditions of the theorems in Section~\ref{bhag:manifoldapprox}, and we will see an ``end point effect'' in the errors, but we find it easy to work with this example because of the ease in computing the various quantities like the volume measure (arc-length) : $d\mu^*=(\sqrt{8}\pi^2)^{-1}dt$.
The target function $f$ is given by
\be\label{helixtargetfn}
f(\x(t)):=\cos(x_1(t)-x_2(t)+x_3(t)/2)=\cos(\cos(\pi t)-\sin(\pi t)-\pi t/2), \qquad 0\le t\le 2\pi.
\ee

\begin{uda}\label{uda:difficult_noise}
{\rm
We consider data of the form
\be\label{difficultdata}
\mathcal{F}(\y,\epsilon):=f(\y+\epsilon)\exp(1.125),
\ee
where $\epsilon$ is a random normal variable with mean $0$ and standard deviation $1.5$. The factor $\exp(1.125)$ ensures that the expected value of $\mathcal{F}$ is $f$. 
This example illustrates a multiplicative noise as well as additive noise.
We may also consider this to be an example where every point $\y$ on the helix is perturbed by a normal noise with mean $0$ and standard deviation $1$, although we cannot deal directly with the perturbed points in the calculation of $\widehat{F}_{n,\alpha}$. 
We took $M= 256$, $n=64$, $\alpha=1$. The results are  reported in Figure~\ref{fig:difficultnoise} on one trial, as well as the average of  $\widehat{F}_{n,\alpha}$ over 100 trials.
\begin{figure}[ht]
\begin{center}
\begin{minipage}{0.3\textwidth}
\includegraphics[width=\textwidth]{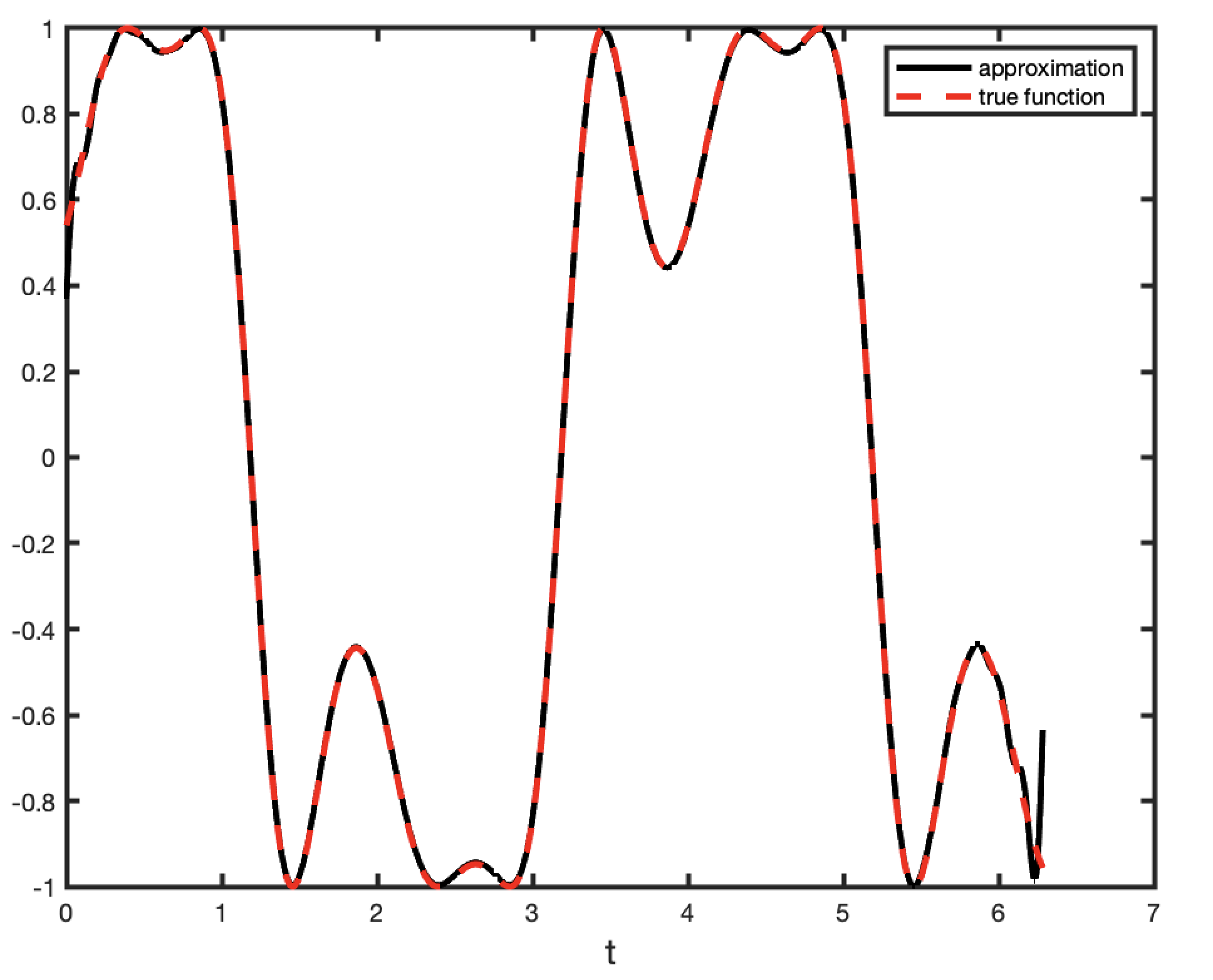} 
\end{minipage}
\begin{minipage}{0.3\textwidth}
\includegraphics[width=\textwidth]{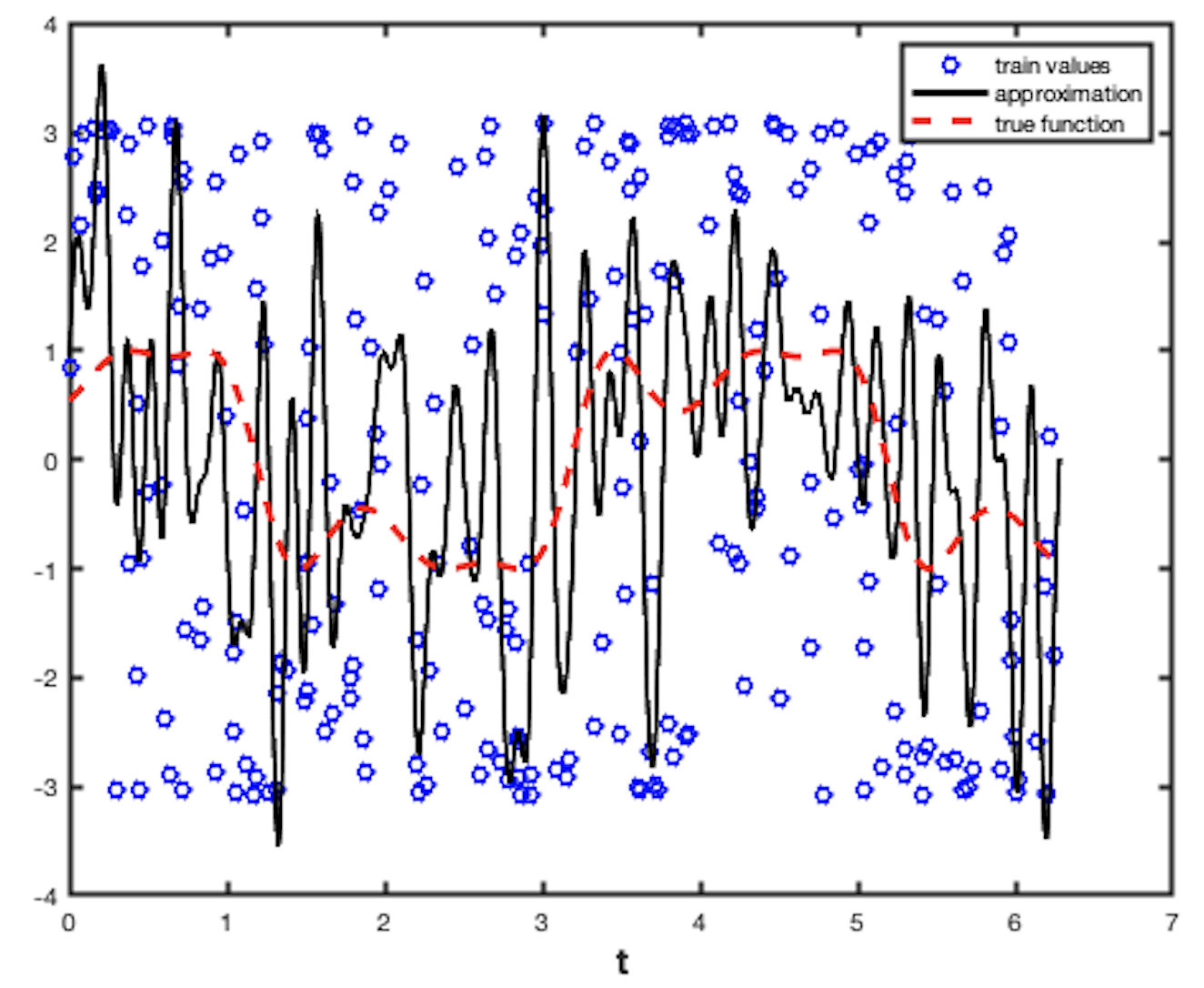} 
\end{minipage}
\begin{minipage}{0.3\textwidth}
\includegraphics[width=\textwidth]{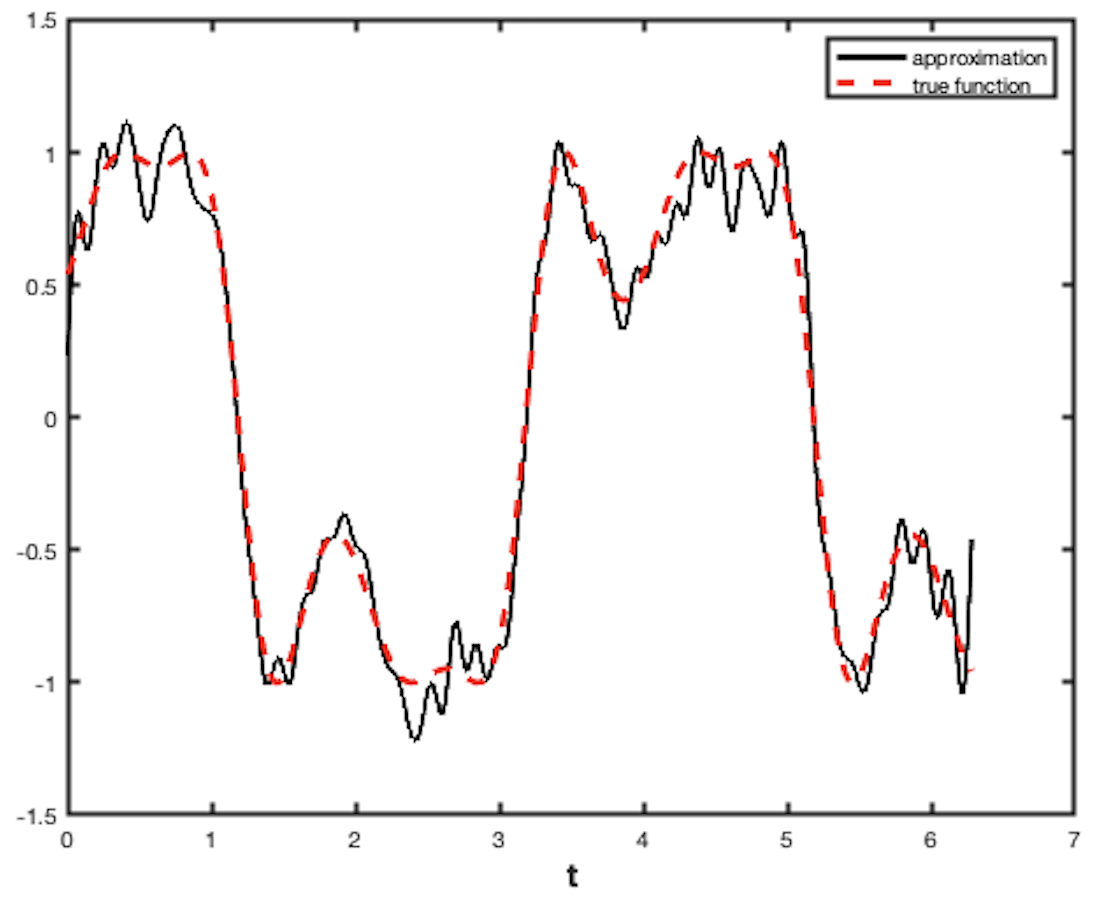} 
\end{minipage}
\end{center}
\caption{In all figures, black continuous line is the approximation, red dashed line is the target function \eqref{helixtargetfn}.
 Left: Reconstruction without noise using $256$ random training points, 2048 equidistant test points, Middle: Estimate in one trial, 256 random training points (blue dots) according to \eqref{difficultdata}, 2048 equidistant test points,
Right:  Average of the estimates in 100 trials, 256 random training points plus noise each, 2048 equidistant test points}
\label{fig:difficultnoise}
\end{figure}
\qed}
\end{uda}

\begin{uda}\label{uda:helixmultnoise}
{\rm
We consider data of the form
\be\label{additivedata}
\mathcal{F}(\y,\epsilon):=f(\y)+\epsilon,
\ee
where $\epsilon$ is a random normal variable with mean $0$ and standard deviation $0.3$. We take $M=1024$, and $M$ samples of $\y$ distributed uniformly according to $\mu^*$. 
We take $n=64$, $\alpha=1$, and 
compute the quantity $\widehat{F}_{64,1}(Y,\x)$ for $\x=\x(t)$, where $t$ ranges over $2048$ equidistant points on $[0,2\pi]$. 
The results are shown in Figure~\ref{fig:additivenoise}.
\begin{figure}[ht]
\begin{center}
\begin{minipage}{0.3\textwidth}
\includegraphics[width=\textwidth]{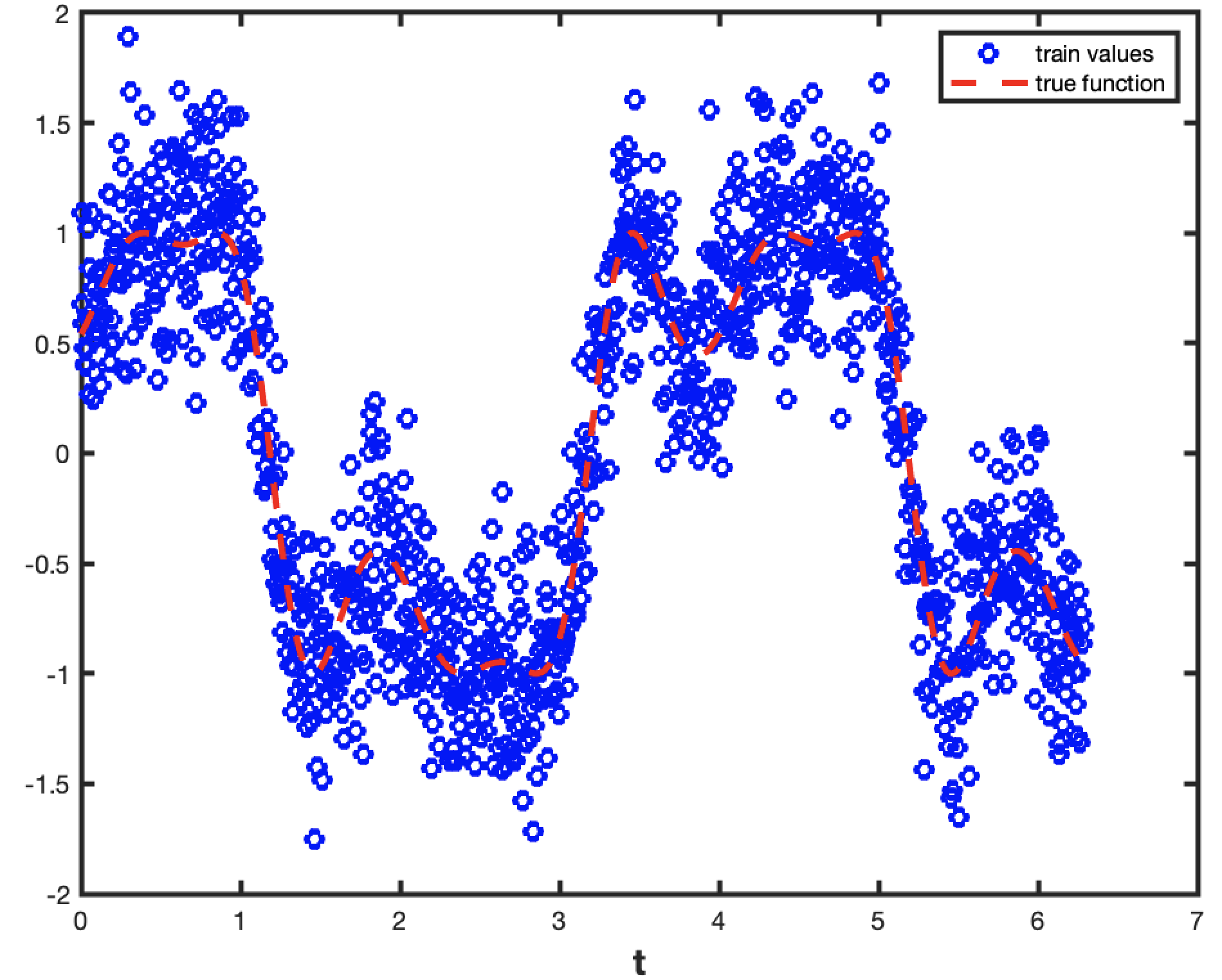} 
\end{minipage}
\begin{minipage}{0.3\textwidth}
\includegraphics[width=\textwidth]{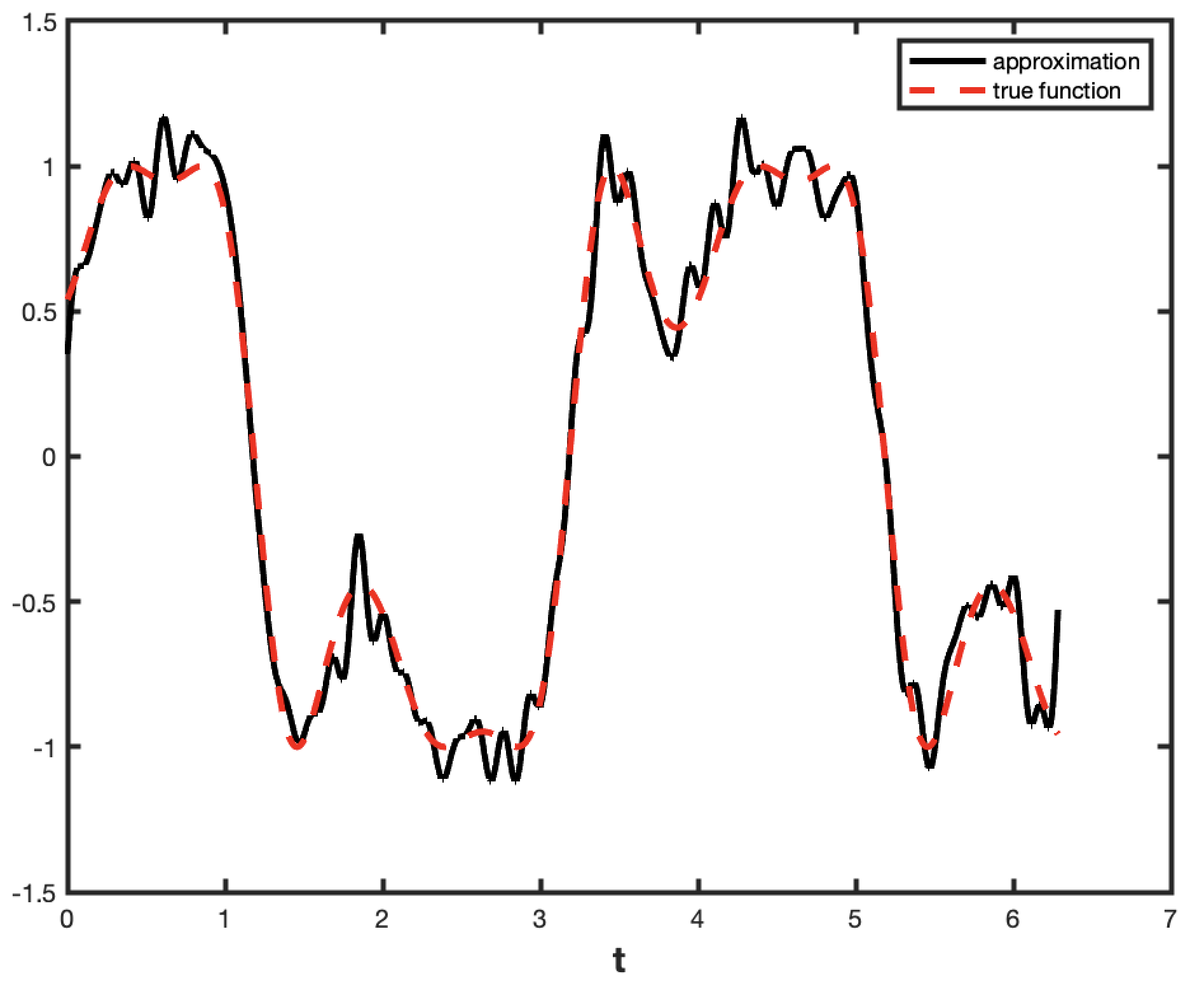} 
\end{minipage}
\begin{minipage}{0.3\textwidth}
\includegraphics[width=\textwidth]{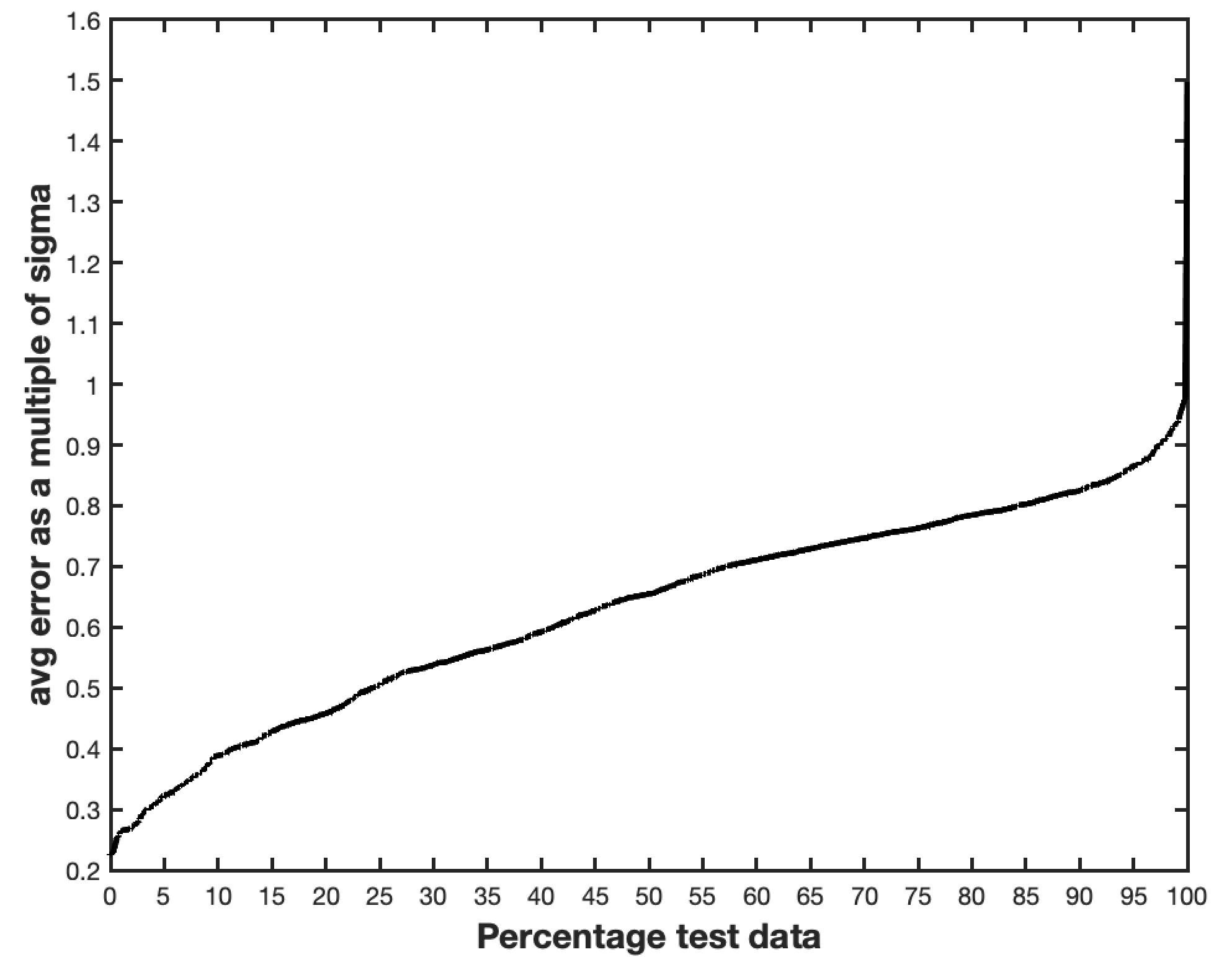} 
\end{minipage}
\end{center}
\caption{With $f$ and $\mathcal{F}$ as in \eqref{helixtargetfn} and \eqref{additivedata} respectively, $M=1024$, $n=64$, $\alpha=1$. On the $x$ axis are 2048 equidistant samples on $[0,2\pi]$. Left: The function $f$ in red, the sampled values $\mathcal{F}$ in blue dots for one trial, Middle: The function $f$ in red, the reconstruction $\widehat{F}_{64,1}$ for one trial in black, Right: A cumulative histogram of errors over 50 trials, the point $(p,y)$ signifies that the error is $0.3y$ at $p\%$ of the test data.}
\label{fig:additivenoise}
\end{figure}
\qed}
\end{uda}


\vskip-1cm
\bhag{Gaussian networks}\label{bhag:gaussnet}
In this section, we describe the consequences of Theorem~\ref{theo:manifoldprob} for Gaussian networks. 
In the case of shallow networks, we can give an explicit construction and error bounds in Section~\ref{bhag:shallow}. 
In the case of deep networks (Section~\ref{bhag:deepnets}), we give only an existence theorem, explaining when the theorem can be described more constructively. 

\subsection{Shallow networks}\label{bhag:shallow}

Since $\mathcal{P}_{m,q}$ and hence $\widetilde{\Phi}_{n,q}$ are even polynomials of degree $<n^2$, $\widetilde{\Phi}_{n,q}(n^{1-\alpha}|\circ|_{2,Q})\in\Pi_n^Q$. We will see in Remark~\ref{rem:phinrem} that  $\widetilde{\Phi}_{n,q}(n^{1-\alpha}|\circ|_{2,Q})=\Phi_{n,q,Q}(\bs 0,n^{1-\alpha}(\circ))$ for a polynomial kernel $\Phi_{n,q,Q}$ on $\RR^Q$. We may then define a \textbf{pre-fabricated} Gaussian network using \eref{poly_to_gauss}
\be\label{specialgauss}
\mathbb{G}_{n,q,Q}^*:=\mathfrak{G}_Q(\Phi_{n,q,Q}(\bs 0, (n^{1-\alpha}(\circ)_{2,Q})).
\ee
Using Corollary~\ref{cor:poly_gaussian}, we then deduce easily the following theorem about Gaussian networks.
We note again that there is no training involved here. Even though the number of non-linearities in the network in the following theorem is $\O(Mn^{2Q})$, this potentially large number of non-linearities is not as much of a problem as it would be if we were to use an optimization procedure to train the network.

\begin{theorem}\label{theo:shallow_net}
Let \eref{ballmeasure} be satisfied, $\gamma>0$, $\tau$ be a probability distribution on $\XX\times \Omega$ for some sample space $\Omega$ such  the marginal distribution of $\tau$ restricted to $\XX$ is $\nu^*$ with $d\nu^*=f_0d\mu^*$ for some $f_0\in W_\gamma(\XX)$.
Let $\mathcal{F} : \XX\times \Omega\to \RR$ be a bounded function, and $f$  defined by \eref{Fdef} be in $W_\gamma(\XX)$. 
Let  $0< \delta<1$,  $\alpha$ satisfy \eref{alphacond}.
Let $M\ge 1$, $Y=\{(\y_1,\epsilon_1),\cdots,(y_M,\epsilon_M)\}$ be a set of random samples chosen i.i.d. from $\tau$. If  
\be\label{nmdefbis}
M\ge n^{q(2-\alpha)+2\alpha\gamma}\sqrt{\log (n/\delta)}
\ee
we have with $\tau$-probability $\ge 1-\delta$:
\be\label{locprobestbis}
\left\|\frac{1}{M}\sum_{j=1}^M \left(\mathcal{F}(\y_j,\epsilon_j) -f(\circ)\right)\mathbb{G}_{n,q,Q}^*(\circ-\y_j) \right\|_\XX\le c_1\frac{\sqrt{\|f_0\|_\XX}\|\mathcal{F}\|_{\XX\times \Omega}+\|f\|_{W_\gamma(\XX)}}{n^{\alpha\gamma}}.
\ee
In particular, let 
\be\label{festimatorbis}
\mathbb{G}_{n,q,Q}(Y;\mathcal{F})(\x):=\frac{1}{M}\sum_{j=1}^M \mathcal{F}(\y_j,\epsilon_j) \mathbb{G}_{n,q,Q}^*(\x-\y_j), \qquad \x\in\RR^Q.
\ee
If $f_0\equiv 1$, we have with $\tau$-probability $\ge 1-\delta$:
\be\label{locprobest_approx_bis}
\left\|\mathbb{G}_{n,q,Q}(Y;\mathcal{F}) -f\right\|_\XX\le c_1\frac{\|\mathcal{F}\|_{\XX\times \Omega}+\|f\|_{W_\gamma(\XX)}}{n^{\alpha\gamma}}.
\ee
\end{theorem}

\subsection{Deep networks}\label{bhag:deepnets}
The following discussion about the terminology about the deep networks  is based on (almost taken from) the discussion in \cite{dingxuanpap, mhaskar2019analysis}, and elaborates upon the same. In particular,  Figure~\ref{graphpict} is taken from the arxiv version of \cite{dingxuanpap}.

A commonly used definition of a deep network is the following. 
Let $\phi :\RR\to\RR$ be an activation function; applied to a vector $\x=(x_1,\cdots,x_q)$, $\phi(\x)=(\phi(x_1),\cdots,\phi(x_q))$. Let $L\ge 2$ be an integer, for $\ell=0,\cdots,L$, let $q_\ell\ge 1$  be an integer ($q_0=q$),  $T_\ell :\RR^{q_\ell}\to \RR^{q_{\ell+1}}$ be an affine transform, where $q_{L+1}=1$.  A deep network with $L-1$ hidden layers is defined as the compositional function
\be\label{usual_deep}
\x\mapsto T_L(\phi(T_{L-1}(\phi(T_{L-2}\cdots\phi(T_0(\x))\cdots).
\ee
This definition has several shortcomings. 
First, it does not distinguish between a function and the network architecture. As demonstrated in \cite{mhaskar2019analysis}, a function may have more than one compositional representation,  so that the affine transforms and $L$ are not determined uniquely by the function itself. 
Second, this notion does not capture the connection between the nature of the target function and its approximation.
Third, the affine transforms $T_\ell$ define a special directed acyclic graph (DAG). 
It is cumbersome to describe notions of weight sharing, convolutions, sparsity, skipping of layers, etc. in terms of these transforms. 
Therefore, we have proposed in \cite{dingxuanpap} to separate the architecture from the function itself, and describe a deep network more generally as a directed acyclic graph (DAG) architecture.

Let $\mathcal{G}$ be a DAG, with the set of nodes $V\cup \mathbf{S}$, where $\mathbf{S}$ is the set of source nodes, and $V$ that of non-source nodes. 
For each node $v\in V\cup \mathbf{S}$, we denote its in-degree by $d(v)$. 
Associated with each $v\in V\cup \mathbf{S}$ is a compact, connected, Riemmanian submanifold $\XX_v$ of $\RR^{d(v)}$ with dimension $q_v$, metric $\rho_v$ and volume element $\mu^*_v$. We assume further that \eref{ballmeasure} is satisfied with $q_v$ in place of $q$.
Each of the  in-edges to each node in $V\cup \mathbf{S}$ represents an input real variable. 
If $v\in V$, $u\in V\cup \mathbf{S}$, $u$ is called the \textbf{child} of $v$ if there is an edge from $u$ to $v$. 
The notion of the level of a node is defined as follows.
The level of a source node is $0$. 
The level of $v\in V$ is the length of the longest path from the nodes in $\mathbf{S}$ to $v$.

Each node $v$ is supposed to evaluate a function $f_v$ on its input variables, supplied via the in-edges for $v$.  The value of this function is propagated along the out-edges of $v$.
Each of the source nodes obtains an input from some smooth manifold as described in Section~\ref{bhag:manifoldapprox}. 
Other nodes can also obtain such an input, but by introducing dummy nodes, it is convenient to assume that only the source nodes obtain an input from the manifold.

Intuitively, we wish to say that the DAG structure implies a compositional structure for the functions involved; for example, if $u_1,\cdots, u_{d(v)}$ are children of $v$, then the function evaluated at $v$ is $f_v(f_{u_1},\cdots, f_{u_{d(v)}})$. 
To make this meaningful, we have to assume some ``pooling'' operation on the input variables to make sure that the output of the vector valued function
$(f_{u_1},\cdots, f_{u_{d(v)}})$ belongs to $\XX_v$. 
Thus, for example, if the domain of $f_v$ is the cube $[-1,1]^{d(v)}$, some clipping operation is required; if the domain is the torus in $d(v)$ dimensions then some standard substitutions need to be made (e.g., \cite{mhaskar2019analysis}). 
We do not know how to specify the pooling operation in the general case of an unknown manifold, but
 assume that this pooling operation $\pi_v :\RR^{d(v)}\to \XX_v$ has the following property: For any two sets of functions $\{f_v\in C(\XX_v)\}_{v\in V}$, $\{g_v\in C(\XX_v)\}_{v\in V}$,
\be\label{pooling_cond}
\rho_v\left(\pi_v(f_{u_1}(\x_{u_1}), \cdots, f_{u_{d(v)}}(\x_{u_{d(v)}})), \pi_v(g_{u_1}(\x_{u_1}), \cdots, g_{u_{d(v)}}(\x_{u_{d(v)}}))\right)\le c(v)\sum_{k=1}^{d(v)}\|f_{u_k}-g_{u_k}\|_{\XX_{u_k}}, \qquad v\in V.
\ee

A $\mathcal{G}$-function is defined to be a set of functions $\{f_v\}_{v\in V\cup\mathbf{S}}$ such that each $f_v\in C(\XX_v)$, and if $v\in V$, $u_1,\cdots, u_{d(v)}$ are children of $v$, then the function evaluated at $v$ is $f_v(\pi_v(f_{u_1},\cdots, f_{u_{d(v)}}))$. 
The individual functions $f_v$ will be called \textit{constituent functions}.

For example, the DAG $\mathcal{G}$ in Figure~\ref{graphpict} (\cite{dingxuanpap}) represents the compositional function
\bea\label{gfuncexample}
f^*(x_1,\cdots, x_9)&=&
h_{19}(h_{17}(h_{13}(h_{10}(x_1,x_2,x_3, h_{16}(h_{12}(x_6,x_7,x_8,x_9))), h_{11}(x_4,x_5)), \nonumber\\ 
&& \qquad\qquad h_{14}(h_{10},h_{11}), h_{16}), h_{18}(h_{15}(h_{11},h_{12}),h_{16})).
\eea
The $\mathcal{G}$-function is $\{h_{10},\cdots,h_{19}= f^*\}$.

\begin{figure}[!h]
\begin{center}
\includegraphics[scale=0.25]{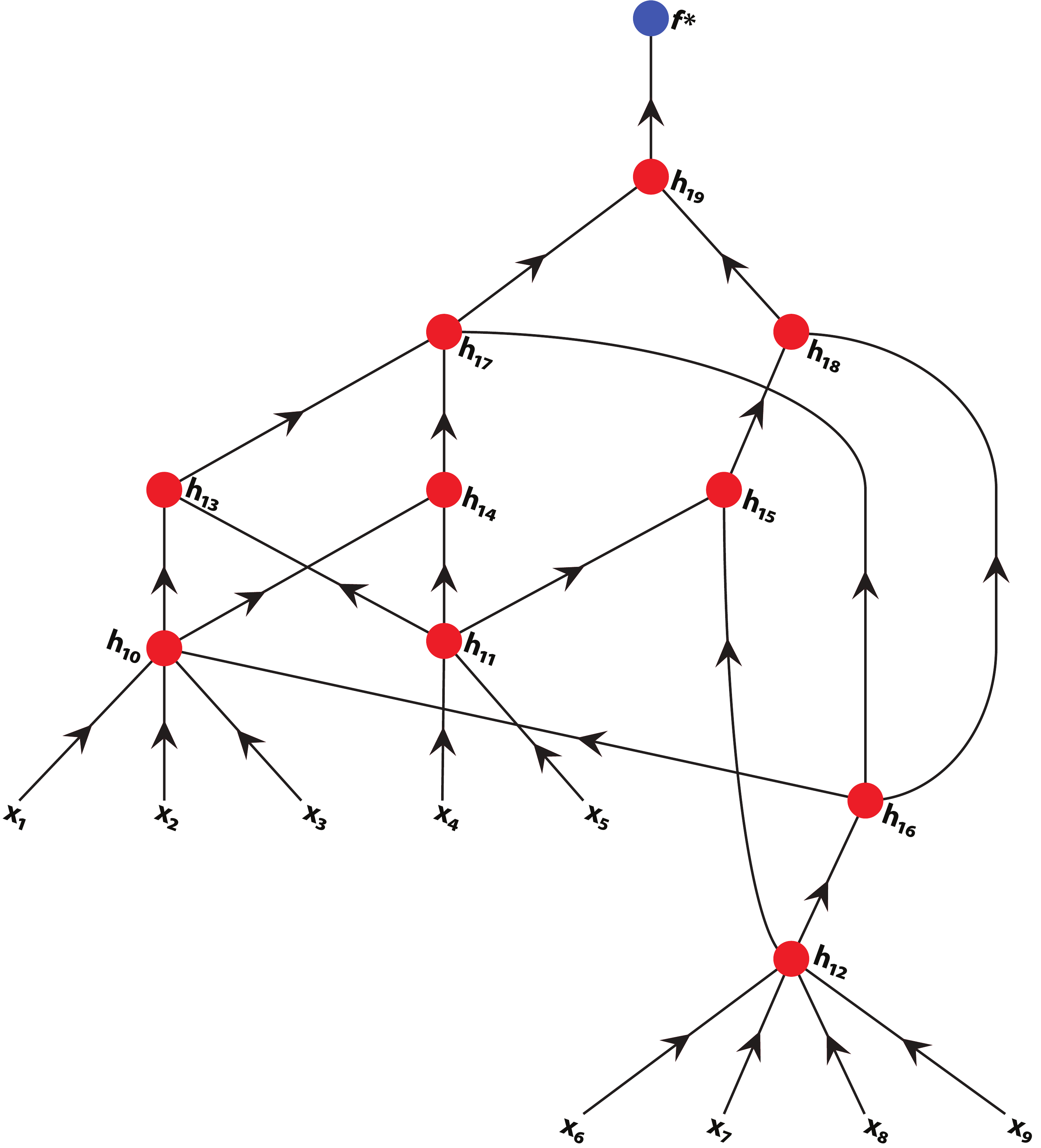} 
\end{center}

\caption{An example of a $\mathcal{G}$-function ($f^*$ given in \eref{gfuncexample}). The vertices of the DAG $\mathcal{G}$ are the channels of the network. The  input to the various channels is indicated by the in--edges of the  nodes, and the output of the sink node $h_{19}$ indicates the output value of the $\mathcal{G}$-function, $f^*$ in this example.}
\label{graphpict}
\end{figure}

We assume that there is only one sink node, $v^*$ (or $v^*(\mathcal{G})$) whose output is denoted by $f_{v^*}$ (the \emph{target function}). 
Technically, there are two functions involved here: one is the final output as a function of all the inputs to all source nodes, the other is the final output as a function of the inputs to the node $v^*$. 
We will use the symbol $f_{v^*}$ to denote both with comments on which meaning is intended when we feel that it may not be clear from the context. A similar convention is followed with respect to each of the constituent functions as well. 
For example, in the DAG of Figure~\ref{graphpict},
the function $h_{15}$ can be thought of both as a function of two variables, namely the outputs of $h_{11}$ and $h_{12}$ as well as a function of six variables $x_4,\cdots, x_9$.
In particular, if each constituent function is a neural network, $h_{15}$ is a shallow network receiving two inputs. 

We define the notion of the variables ``seen'' by a node. 
If $u\in\mathbf{S}$, then these are the variables input to $u$. 
Let $v\in V$, and $u_1,\cdots, u_{d(v)}$ be the children of $v$.
If $\x_1,\cdots,\x_{d(v)}$ are the inputs seen by $u_1,\cdots, u_{d(v)}$, then the inputs seen by $v$ are $(\x_1,\cdots, \x_{d(v)})$, where the order is respected. 
For example, consider the function
$$
f^*(x_1,x_2,x_3,x_4)=f(f_1(x_1,x_2), f_2(x_4,x_2), f_3(x_3,x_1)).
$$
The inputs seen by the leaves $f_1$, $f_2$, $f_3$ are $(x_1,x_2)$, $(x_4,x_2)$, $(x_3,x_1)$ respectively (not $(x_1,x_2)$, $(x_2,x_4)$, $(x_1,x_3)$). The inputs seen by $f^*$ are $(x_1,x_2,x_3,x_4)$.

The following theorem enables us to ``lift'' a theorem about shallow networks to that about deep networks. 
\begin{theorem}\label{theo:good_propogation}
Let $\mathcal{G}$ be a DAG as described above, $\{f_v\}_{v\in V\cup\mathbf{S}}$, $\{g_v\}_{v\in V\cup\mathbf{S}}$ be $\mathcal{G}$-functions, and
\be\label{shallowest}
\|f_v-g_v\|_{\XX_v}\le \varepsilon, \qquad v\in V\cup\mathbf{S}.
\ee
Further assume that for each $v\in V$, $f_v\in \Lip(\XX_v)$, with $L=\max_{v\in V}\|f_v\|_{\Lip(\XX_v)}$.
Then for the target function, thought of as a compositional function of all the input variables $\x$ to all the nodes in $\mathbf{S}$, we have
\be\label{deepest}
|f_{v^*}(\x)-g_{v^*}(\x)| \le c(L,\mathcal{G})\varepsilon.
\ee
\end{theorem} 

Theorem~\ref{theo:good_propogation} allows us to lift Theorem~\ref{theo:shallow_net} to deep networks. 
In general, we do not know the constituent functions.
Also, for any given function and a DAG structure, it may not be possible to devise an algorithm to find the constituent functions uniquely.
For example, $(\cos^2 x)^2$ and $(1/4)(1+\cos(2x))^2$
both have the structures $g_1(g_2(x))$ or $f_1(f_2(x))$, both representing 
the same DAG but with different constituent functions.
Thus, even if we may assume that the noise occurs only in the approximation of the target function at the sink node and not in the constituent functions, it seems to be an extremely difficult problem to determine theoretically for any target function what the optimal DAG structure and the input/output for the constituent functions ought to be.
Therefore, we have to state our theorem for deep networks 
 only as an existence theorem, in the non-noisy case, not to complicate the notations too much.
 We assume also that at each node $v$, the input data is distributed according to the volume measure of $\XX_v$.
\begin{theorem}\label{theo:deep_gaussian}
Let $\mathcal{G}$ be a DAG as described above, $\{f_v\}$ be a $\mathcal{G}$-function, and we assume that each of the constituent functions $f_v\in W_\gamma(\XX_v)\cap \Lip(\XX_v)$ for some $\gamma>0$, $\alpha$ satisfy \eqref{alphacond}.
Let  $n\ge 1$.
Then there exists a $\mathcal{G}$-function $\{g_v\}$ such that each $g_v$ is a Gaussian network constructed using $\O(n^{q_v(2-\alpha)+2\alpha\gamma}\log n)$ samples of its inputs, such that for any $\x$ seen by $v^*$,
\be\label{deepnetest}
|f_{v^*}(\x)-g_{v^*}(\x)| \le c(L,\mathcal{G})n^{-\alpha\gamma}.
\ee
\end{theorem}

\bhag{Background on weighted polynomials}\label{bhag:backwtpoly}
\subsection{Weighted polynomials}\label{bhag:wtedpoly}
A good preliminary source of many identities regarding Hermite polynomials is the book \cite{szego} of Szeg\"o or the Bateman manuscript \cite{batemanvol2}.

We denote the class of all univariate algebraic polynomials of degree $<n$ by $\mathbb{P}_n$.
The orthonormalized Hermite polynomial $h_k$ of degree $k$ is defined recursively by \eref{recurrence}.
With $\psi_k(x)=h_k(x)\exp(-x^2/2)$, one has the orthogonality relation for $k, j\in \ZZ_+$,
\be\label{uniortho}
\int_\RR \psi_k(x)\psi_j(x)dx=\left\{\begin{array}{ll}
1, & \mbox{if $k=j$, }\\
0, & \mbox{if $k\not=j$.}
\end{array}\right.
\ee
 
Using \eqref{recurrence}, it is easy to deduce by induction that
\be\label{reduceprojdetails1}
\psi_\ell(0)=\left\{\begin{array}{ll}
\disp \pi^{-1/4}(-1)^{\ell/2}\frac{\sqrt{\ell!}}{2^{\ell/2}(\ell/2)!},&\mbox{if $\ell$ is even},\\[1ex]
0, &\mbox{if $\ell$ is odd},
\end{array}\right.
\ee

 The Hermite polynomial $h_m$ has $m$ real and simple zeros $x_{k,m}$. Writing
\be\label{coatesnum}
 \lambda_{k,m}:=\left(\sum_{j=0}^{m-1} h_j(x_{k,m})^2\right)^{-1},
\ee
it is well known (cf. \cite[Section~3.4]{szego}) that
\be\label{uniquad}
\sum_{k=1}^m \lambda_{k,m}P(x_{k,m})= \int_\RR P(x)\exp(-x^2)dx, \qquad P\in \mathbb{P}_{2m}.
\ee
It is also known (cf. \cite[Theorem~8.2.7]{mhasbk}, applied with $p=2$, $b=0$) that
\be\label{uniquadbds}
\sum_{k=1}^m \lambda_{k,m}\exp(x_{k,m}^2) \le cm^{1/2}.
\ee
The Mehler formula \cite[Formula~(6.1.13)]{andrews_askey_roy} states that
\be\label{unimehler}
\sum_{j=0}^\infty \psi_j(y)\psi_j(z)w^j=
 \frac{1}{\sqrt{\pi (1-w^2)}}\exp\left(\frac{2yzw-(y^2+z^2)w^2}{1-w^2}\right)\exp(-(y^2+z^2)/2), \qquad y, z\in\RR, \ w\in\CC,\ 
|w|<1.
\ee

Next, we introduce and review the properties of Hermite polynomials in the multivariate setting. 
We will need to use spaces with many different dimensions. 
Therefore, in this section,  we will use the symbol $d$ to denote a generic dimension, which will be replaced later by $q$, $Q$, $q_v$, etc.

If $d\ge 2$ is an integer, we define Hermite polynomials on $\RR^d$ using tensor products. We adopt the notation $\x=(x_1,\cdots, x_d)$. The orthonormalized Hermite function is defined by
\be\label{multihermite}
\psi_\k(\x):=\prod_{j=1}^d\psi_{k_j}(x_j).
\ee
In general, when univariate notation is used in multivariate context, it is to be understood in the tensor product sense as above; e.g., $\k!=\prod_{j=1}^d (k_j!)$,
$\x^\k=\prod_{j=1}^dx_j^{k_j}$, etc. The notation $|\cdot|_{p,d}$ will denote the  $\ell^p$ norm on $\RR^d$.

For any set $A\subset \RR^d$ and $f: A\to\RR$, we denote by $C(A)$ the space of all uniformly continuous and bounded functions on $A$, with the norm $\|f\|_A=\sup_{\x\in A}|f(\x)|$. The space $C_0(A)$ is the subspace of all $f\in C(A)$ vanishing at infinity.

We will often use (without mentioning it explicitly) the fact deduced from the univariate bounds proved in \cite{askey1965mean}  that
\be\label{hermitebd}
|\psi_\k(\x)|\le c.
\ee
We will denote by $\Pi_n^d$ the span of $\{\psi_\k : \sqrt{|\k|_{1,d}}<n\}$ and by $\mathbb{P}_n^d$ the space of all algebraic polynomials of total degree $<n$. 
Thus, if $P\in \Pi_n^d$, then $P(\x)=R(\x)\exp(-|\x|_{2,d}^2/2)$ for some $R\in \mathbb{P}_{n^2}^d$.
The following proposition lists a few important properties of these spaces (cf. \cite{mhasbk, gaussbern, mohapatrapap}).
\begin{prop}\label{prop:wtpoly}
Let $n >0$, $P\in\Pi_n^d$. \\
{\rm (a) (Infinite-finite range inequality)} For any $\delta>0$, there exists $c=c(\delta)$ such that
\be\label{rangeineq}
\|P\|_{\RR^d\setminus [-\sqrt{2}n(1+\delta), \sqrt{2}n(1+\delta)]^d} \le c_1e^{-cn^2}
\|P\|_{[-\sqrt{2}n(1+\delta), \sqrt{2}n(1+\delta)]^d}
\ee
{\rm (b) (MRS identity)} We have
\be\label{mrsidentity}
\|P\|_{\RR^d}=\|P\|_{[-\sqrt{2}n, \sqrt{2}n]^d}.
\ee
{\rm (c) (Bernstein inequality)} There is a positive constant $B$ depending only on $d$ such that
\be\label{bernineq}
\big\||\nabla P|_{2,d}\big\|_{\RR^d} \le Bn\|P\|_{\RR^d}.
\ee
\end{prop}

Let $m\ge 1$. For a multi-integer $\j$, $1\le \j\le m$, we write $\x_{\j,m,d}:=(x_{j_1,m},\cdots,x_{j_d,m})$, and $\lambda_{\j,m,d}:=\prod_{\ell=1}^d \lambda_{j_\ell,m}$. 
We observe further that if $P_1, P_2\in \Pi_m^d$, then $P_1(\x)P_2(\x) =R(\x)\exp(-|\x|^2)$ for some $R\in \mathbb{P}_{2m^2}^d$. 
Therefore, \eref{uniquad} and \eref{uniquadbds} lead to the following fact, which we formulate as a proposition.

\begin{prop}\label{prop:multi_quad}
For $m\ge 1$, we have
\be\label{multi_quad}
\sum_{1\le \j\le m^2} \lambda_{\j,m^2,d}\exp(|\x_{\j,m^2,d}|_{2,d}^2)P_1(\x_{\j,m^2,d})P_2(\x_{\j,m^2,d})
=\int_{\RR^d}P_1(\x)P_2(\x)d\x, \qquad P_1,P_2\in \Pi_m^d,
\ee
and
\be\label{mulit_quad_bds}
\sum_{1\le \j\le m^2} \lambda_{\j,m^2,d}\exp(|\x_{\j,m^2,d}|_{2,d}^2) \le cm^{d/2}.
\ee
\end{prop}

\subsection{Applications of Mehler identity}\label{bhag:mehler}
The Mehler identity for multivariate Hermite polynomials is expressed conveniently by writing
\be\label{projdef}
\mathsf{Proj}_{m,d}(\x,\y):=\sum_{|\k|_{1,d}=m}\psi_\k(\x)\psi_\k(\y).
\ee
Using the univariate Mehler identity \eref{unimehler}, it is then easy to deduce  that for $w\in\CC$, $|w|<1$,
\be\label{mehler}
\begin{aligned}
 \sum_{\k\in\ZZ^d} \psi_\k(\x)&\psi_\k(\y)w^{|\k|_{1,d}}= \sum_{m=0}^\infty w^m \mathsf{Proj}_{m,d}(\x,\y)\\
 &=\frac{1}{(\pi(1-w^2))^{d/2}}\exp\left(\frac{4w\x\cdot\y-(1+w^2)(|\x|_{2,d}^2+|\y|_{2,d}^2)}{2(1-w^2)}\right)\\
 &= \frac{1}{(\pi(1-w^2))^{d/2}}\exp\left(-\frac{1+w}{1-w}\frac{|\x-\y|_{2,d}^2}{4}-\frac{1-w}{1+w}\frac{|\x+\y|_{2,d}^2}{4}\right)\\
 &=\frac{1}{(\pi(1-w^2))^{d/2}}\exp\left(-\frac{1+w^2}{2(1-w^2)}\left|\x-\frac{2w}{1+w^2}\y\right|_{2,d}^2\right)\exp\left(-\frac{1-w^2}{2(1+w^2)}|\y|_{2,d}^2\right).
\end{aligned}
\ee
We note an identity \eqref{psiatzero} which follows immediately from \eqref{mehler} by setting $\x=\y=\bs 0$. 
For integer $d$ (not necessarily positive), we define the sequence $D_{d;r}$ by 
\be\label{reduceprojdetails2}
D_{d;r}:=\left\{\begin{array}{ll}
\disp \pi^{-d/2}(-1)^{r/2}\frac{\Gamma(1-d/2)}{\Gamma(1-d/2-r/2)(r/2)!}, &\mbox{ if $r$ is even, $d\le 0$,}\\[1ex]
\disp\pi^{-d/2}\frac{\Gamma(d/2+r/2)}{\Gamma(d/2) (r/2)!}, &\mbox{if $r$ is even, $d\ge 1$,}\\[1ex]
0, &\mbox{if $r$ is odd}.
\end{array}\right.
\ee
This sequence is chosen so as to satisfy
\be\label{binomialid}
\pi^{-d/2}(1-w^2)^{-d/2}=\sum_{r=0}^\infty D_{d;r}w^r, \qquad |w|<1.
\ee
Using the Mehler identity \eref{mehler}, we deduce that for any integer $d\ge 1$
\be\label{psiatzero}
\sum_{r=0}^\infty w^{2r}\sum_{|\k|_{1,d}=2r}|\psi_\k(\bs{0})|^2=(\pi (1-w^2))^{-d/2}=\pi^{-d/2}\sum_{r=0}^\infty  \frac{\Gamma(d/2+r)}{\Gamma(d/2) r!}w^{2r}=\sum_{\ell=0}^\infty D_{d;\ell}w^\ell.
\ee

In this section, we point out the invariance and localization properties of certain kernels using the Mehler identity.

\subsubsection{Rotation invariance}\label{bhag:invariance}
An interesting consequence of the Mehler identity is that the projection $\mathsf{Proj}_{m,d}$ is invariant under rotations. For  $d\ge 2$ and any $\x,\y\in\RR^d$, we may therefore use an appropriate rotation to write 
\be\label{reduceproj1}
\mathsf{Proj}_{m,d}(\x,\y)=\sum_{j=0}^m \mathsf{Proj}_{j,2}((|\x|_{2,d},0),(|\y|_{2,d}\cos\theta, |\y|_{2,d}\sin\theta))\sum_{|\k|_{1,{d-2}}\le m-j}|\psi_\k(\bs 0)|^2,
\ee
where
$\cos\theta=\x\cdot\y/(|\x||\y|)$, with obvious modifications if $\y=\bs 0$.
Hence, we obtain from \eqref{reduceproj1} and \eqref{psiatzero} (used with $d-2$ in place of $d$), 
\be\label{reduceproj}
\mathsf{Proj}_{m,d}(\x,\y)=\sum_{j=0}^m  \mathsf{Proj}_{j,2}((|\x|_{2,d},0),(|\y|_{2,d}\cos\theta, |\y|_{2,d}\sin\theta))D_{d-2;m-j}.
\ee
In the case when $d=1$, \eqref{reduceproj1} takes the form
\be\label{reduceprojd1}
\mathsf{Proj}_{m,1}(x,y)=\psi_m(|x|)\psi_m(|y|\cos\theta),
\ee
where $\cos\theta=xy/(|x||y|)=\mathrm{sgn}(xy)$, ($\mathrm{sgn}(0)=0$).

Let $Q\ge q\ge 1$ be integers. We can  extend the definition of $\mathsf{Proj}_{m,q}$ to $\x,\y\in\RR^Q$ by
\be\label{genrepkerndef}
\mathsf{Proj}_{m,q,Q}(\x,\y):=\begin{cases}
\disp\sum_{j=0}^m  \mathsf{Proj}_{j,2}((|\x|_{2,Q},0),(|\y|_{2,Q}\cos\theta, |\y|_{2,Q}\sin\theta))D_{q-2;m-j}, & \mbox{if  $q\ge 2$},\\[1ex]
\psi_m(|\x|_{2,Q})\psi_m(|\y|_{2,Q}\cos\theta), & \mbox{if  $q=1$}.
\end{cases}
\ee

The relationship between $\mathsf{Proj}_{m,q,Q}$ and $\mathsf{Proj}_{m,Q}$, both defined on $\RR^Q$ is given by the following proposition.
\begin{prop}\label{prop:Q_to_q_reduce}
Let $Q> q\ge 2$ be integers. Let $m\ge 0$, and $\x,\y\in\RR^Q$. \\
{\rm (a)} We have
\be\label{q_to_Q_extension}
\mathsf{Proj}_{m,Q}(\x,\y)=\pi^{(q-Q)/2}\sum_{\ell=0}^{\lfloor m/2\rfloor}  \binom{(Q-q)/2+\ell-1}{\ell}\mathsf{Proj}_{m-2\ell,q,Q}(\x,\y).
\ee 
{\rm (b)} We have
\be\label{Q_to_q_reduction}
\mathsf{Proj}_{m,q,Q}(\x,\y)=\pi^{(Q-q)/2}\sum_{\ell=0}^{\lfloor m/2\rfloor} (-1)^\ell \binom{(Q-q)/2}{\ell}\mathsf{Proj}_{m-2\ell,Q}(\x,\y).
\ee
Hence, $\mathsf{Proj}_{m,q,Q}(\x,\y)$ is a weighted polynomial in $\Pi_m^Q$ as a function of $\x$ and $\y$. \\[1ex]
(c)  If $\x$ is a scalar multiple of $\y$, then
 \eqref{q_to_Q_extension} and \eqref{Q_to_q_reduction} both hold also when $q=1$.
\end{prop}

\begin{proof}\ 
In this proof, let $\x'=(|\x|_{2,Q},0, \underbrace{0,\cdots,0}_{\mbox{$q-2$ times}}\!\!\!)$, $\y'=(|\y|_{2,Q}\cos\theta, |\y|_{2,Q}\sin\theta, \underbrace{0,\cdots,0}_{\mbox{$q-2$ times}}\!\!\!)$. 
In view of \eref{reduceproj1}, we observe that
\be\label{pf1eqn3}
\mathsf{Proj}_{m,q,Q}(\x,\y)=\mathsf{Proj}_{m,q}(\x',\y').
\ee
Further, $|\x-\y|_{2,Q}=|\x'-\y'|_{2,q}$, $|\x+\y|_{2,Q}=|\x'+\y'|_{2,q}$. Therefore, the Mehler identity \eref{mehler} shows that
\be\label{pf1eqn4}
\begin{aligned}
\sum_{m=0}^\infty w^m \mathsf{Proj}_{m,Q}(\x,\y)
 &= \frac{1}{(\pi(1-w^2))^{Q/2}}\exp\left(-\frac{1+w}{1-w}\frac{|\x-\y|_{2,Q}^2}{4}-\frac{1-w}{1+w}\frac{|\x+\y|_{2,Q}^2}{4}\right)\\
 &=\frac{1}{(\pi(1-w^2))^{(Q-q)/2}}\sum_{m=0}^\infty w^m \mathsf{Proj}_{m,q}(\x',\y')\\
 &=\frac{1}{(\pi(1-w^2))^{(Q-q)/2}}\sum_{m=0}^\infty w^m \mathsf{Proj}_{m,q,Q}(\x,\y).
 \end{aligned}
\ee
We now recall the McClaurin expansion for $(1-w^2)^{-(Q-q)/2}$ (cf. \eref{binomialid}), multiply the two power series using the Cauchy-Leibnitz formula, and compare the coefficients to arrive at \eref{q_to_Q_extension}.
Part (b) is proved similarly by observing that
\be\label{pf1eqn6}
\sum_{m=0}^\infty w^m \mathsf{Proj}_{m,q,Q}(\x,\y)=\pi^{(Q-q)/2}(1-w^2)^{(Q-q)/2}\sum_{m=0}^\infty w^m \mathsf{Proj}_{m,Q}(\x,\y).
\ee
If $\x$ is a scalar multiple of $\y$, then $\sin\theta=0$, so that
$\x'=(|\x|_{2,Q}, \underbrace{0,\cdots,0}_{\mbox{$q-1$ times}}\!\!\!)$, $\y'=(|\y|_{2,Q}\cos\theta,  \underbrace{0,\cdots,0}_{\mbox{$q-1$ times}}\!\!\!)$. Part (c) is then proved using the same calculations as above.
\end{proof}

\begin{rem}\label{rem:fastproj}
{\rm
Clearly, for every $\x,\y\in\RR^Q$, $\mathsf{Proj}_{m,Q,Q}(\x,\y)=\mathsf{Proj}_{m,Q}(\x,\y)$, $\mathsf{Proj}_{m,q,Q}(\x,\y)=\mathsf{Proj}_{m,q,Q}(-\x,-\y)$ and the kernel $(\x,\y)\mapsto \mathsf{Proj}_{m,q,Q}(\bs 0, \x-\y)$ is both rotation invariant and translation invariant.
Using \eqref{reduceproj1}, \eref{reduceprojdetails1},  and \eqref{psiatzero} (used with $d=q-1$) show that for all $q\ge 1$ and $m\in \ZZ_+$, $\mathsf{Proj}_{2m-1,q,Q}(\bs 0, \x)=0$, and
\be\label{fastvsusualproj}
\mathsf{Proj}_{2m,q,Q}(\bs 0, \x)=\mathsf{Proj}_{2m,q}(\bs 0, \x')=\sum_{j=0}^{2m} \psi_j(|\x|_{2,Q})\psi_j(0)\sum_{|\k|_{1,{q-1}}\le 2m-j}|\psi_\k(\bs 0)|^2=\mathcal{P}_m(|\x|_{2,Q}).
\ee
\qed}
\end{rem}

\subsubsection{Localized kernels}\label{bhag:lockern}

In this section, we recall the localization properties of certain kernels.
In the sequel, $H : [0,\infty)\to [0,1]$ is a fixed, infinitely differentiable function, with $H(t)=1$ if $0\le t\le 1/2$, $H(t)=0$ if $t\ge 1$. All constants may depend upon $H$ as well.
We define
\be\label{summkerndef}
\Phi_{n,d}(H;\x,\y):=\Phi_{n,d}(\x,\y):=\sum_{\k\in\ZZ_+^d}H\left(\frac{\sqrt{|\k|_{1,d}}}{n}\right)\psi_\k(\x)\psi_\k(\y)=\sum_{m=0}^{n^2} H\left(\frac{\sqrt{m}}{n}\right)\mathsf{Proj}_{m,d}(\x,\y), \qquad \x,\y\in\RR^d.
\ee

Using Mehler identity and the Tauberian theorem in \cite[Theorem~4.3]{tauberian}, we proved in \cite[Lemma~4.1]{hermite_recovery} the following proposition.

\begin{prop}\label{prop:kernloc}
For $n\ge 1$, $\x,\y\in\RR^d$, we have
\be\label{basic_kernloc}
|\Phi_{n,d}(\x,\y)| \le \frac{cn^d}{\max(1,(n|\x-\y|_{2,d})^S)}.
\ee
In particular,
\be\label{basic_kernnorm}
|\Phi_{n,d}(\x,\y)| \le cn^d,
\ee
and for $1\le p<\infty$,
\be\label{kern_op_norm}
\sup_{\x\in\RR^d}\int_{\RR^d}|\Phi_{n,d}(\x,\y)|^pd\y \le  cn^{d(p-1)}.
\ee
\end{prop}
 
We extend the definition of $\Phi_{n,d}$ as follows. 
Let $Q\ge q\ge 1$ be integers. 
We define
\be\label{modsummkerndef}
\Phi_{n,q,Q}(\x,\y):=\sum_{m=0}^\infty H\left(\frac{\sqrt{m}}{n}\right)\mathsf{Proj}_{m,q,Q}(\x,\y), \qquad \x,\y\in\RR^Q.
\ee

\begin{rem}\label{rem:phinrem}
{\rm
In view of Remark~\ref{rem:fastproj}, the kernel $\widetilde{\Phi}_{n,q}$ defined in \eref{fastkerndef} satisfies $\widetilde{\Phi}_{n,q}(|\x|_{2,Q})=\Phi_{n,q,Q}(\bs 0, \x)$.
In particular, $\widetilde{\Phi}_{n,q}(|\circ|_{2,Q})\in\Pi_n^Q$.
\qed}
\end{rem}

\begin{prop}\label{prop:Q_q_lockern}
 Let $S>Q\ge q\ge 2$ be  integers. 
 The kernel $\Phi_{n,q,Q}(\x,\y)\in \Pi_n^Q$ as a function of $\x$ and $\y$.   
 For $\x,\y\in\RR^Q$, $n=1,2,\cdots$,
\be\label{hermite_localization}
|\Phi_{n,q,Q}(\x,\y)| \le \frac{cn^q}{\max(1,(n|\x-\y|_{2,Q})^S)}.
\ee
In particular,
\be\label{hermite_max_norm}
|\Phi_{n,q,Q}(\x,\y)| \le cn^q.
\ee
If  $\x$ is a scalar multiple of $\y$, then
\be\label{hermite_loc_qeq1}
|\Phi_{n,1,Q}(\x,\y)| \le \frac{cn}{\max(1,(n|\x-\y|_{2,Q})^S)}, \qquad |\Phi_{n,1,Q}(\x,\y)|\le cn.
\ee
\end{prop}

\begin{proof}
Let $\x',\y'$ be as in the proof of Proposition~\ref{prop:Q_to_q_reduce}. Since $\Phi_{n,q,Q}(\x,\y)=\Phi_{n,q}(\x',\y')$, this proposition follows directly from Proposition~\ref{prop:kernloc}.
\end{proof}

\begin{cor}\label{cor:tildephi}
The   kernel $\widetilde{\Phi}_{n,q}$ defined in \eref{fastkerndef} satisfies each of the following properties.
\be\label{tildephiloc}
|\widetilde{\Phi}_{n,q}(|\x|_{2,Q}))| \le \frac{cn^q}{\max(1,(n|\x|_{2,Q})^S)}, \qquad \x\in\RR^q,
\ee
\be\label{tildephibern}
|\widetilde{\Phi}_{n,q}(|\x|_{2,Q}))|\le cn^q, \quad |\widetilde{\Phi}_{n,q}(|\x|_{2,Q}))-\widetilde{\Phi}_{n,q}(|\y|_{2,Q}))| \le cn^{q+1}\left||\x|_{2,Q}-|\y|_{2,Q}\right|, \qquad \x,\y\in\RR^Q.
\ee
\end{cor}

\begin{proof}\ 
The estimate \eqref{tildephiloc} and the first estimate in \eqref{tildephibern} follows from Proposition~\ref{prop:Q_q_lockern} and the fact that $\widetilde{\Phi}_{n,q}(|\x|_{2,Q})=\Phi_{n,q,Q}(\bs 0, \x)$. 
The second estimate in \eqref{tildephibern} follows from the Bernstein inequality \eqref{bernineq} applied with $d=1$ to the univariate polynomial $\widetilde{\Phi}_{n,q}$.
\end{proof}
\subsection{From Hermite polynomials to Gaussian networks}\label{bhag:hermite_to_gaussian}
We discuss in this section the close connection between Hermite polynomials and Gaussian networks.

\begin{prop}\label{prop:poly_to_gaussian}
Let $m\ge 1$, $\k\in\ZZ_+^d$,  and for $|\k|_{1,d} < m^2$, $\x\in\RR^d$,
\be\label{basic_gaussian}
\mathfrak{G}_{\k,m,d}(\x)
:= \left(\frac{3}{2\pi}\right)^{d/2}3^{|\k|_{1,d}/2}\!\!
\sum_{1\le \j\le 2m^2}\!\! \lambda_{\j,2m^2,d}\exp(3|\x_{\j,2m^2,d}|_{2,d}^2/4)
\psi_\k(\x_{\j,2m^2,d})\exp\left(-\left|\x-\frac{\sqrt{3}}{2}
\x_{\j,2m^2,d}\right|_{2,d}^2\right).
\ee
Then
\be\label{basic_gaussian_approx}
\max_{|\k|_{1,d}<m}\left\|\psi_\k-\mathfrak{G}_{\k,m,d}\right\|_{\RR^d} \le cm^{d-2} 3^{-m^2/2}.
\ee
Clearly, the number of neurons in the network $\mathfrak{G}_{\k,m,d}$ is $\O(m^{2d})$.
\end{prop}

\begin{proof}\ 
This proof is the same as that in \cite[Lemma~4.2]{chuigaussian} and \cite[Lemma~4.1]{convtheo}. 
Using the last expression in \eref{mehler} with $w=1/\sqrt{3}$, we obtain
$$
\sum_{\k\in\ZZ_+^d}\psi_\k(\x)\psi_\k(\u)3^{-|\k|_{1,d}/2} = \left(\frac{3}{2\pi}\right)^{d/2}\exp\left(-\left|\x-\frac{\sqrt{3}}{2}\u\right|_{2,d}^2\right)\exp(-|\u|^2/4).
$$
In this proof, we denote by $\nu^*_{m,d}$ the measure that associates the mass $\lambda_{\j,2m^2,d}\exp(|\x_{\j,2m^2,d}|_{2,d}^2)$ with the point $\x_{\j,2m^2,d}$ for $1\le j_1,\cdots,j_d\le 2m^2$.
Therefore, using Proposition~\ref{prop:multi_quad} with $m\sqrt{2}$ in place of $m$, we obtain
\begin{eqnarray*}
\psi_\k(\x)&=&3^{|\k|_{1,d}/2}\left(\frac{3}{2\pi}\right)^{d/2}\int_{\RR^d}\exp\left(-\left|\x-\frac{\sqrt{3}}{2}\u\right|_{2,d}^2\right)\psi_\k(\u)\exp(-|\u|^2/4)d\nu_{m,d}^*(\u)\\
& &\makebox[1in]{} 
-3^{|\k|_{1,d}/2}\int_{\RR^d} \sum_{|\j|_{1,d}\ge 2m^2}\psi_\k(\u)\psi_\j(\x)\psi_\j(\u)3^{-|\j|_{1,d}/2}d\nu_{m,d}^*(\u).
\end{eqnarray*}
The first term on the right hand side above is $\mathfrak{G}_{\k,m,d}$. The second term is estimated using \eref{hermitebd} and \eref{mulit_quad_bds} (applied with $m\sqrt{2}$ in place of $m$) exactly as in the proof of \cite[Lemma~4.2]{chuigaussian}. We omit the details.
\end{proof}

 The following corollary is easy to deduce (cf. \cite[Proposition~4.1]{convtheo}).
 If $P=\sum_{|\k|_{1,d}<m^2}b_\k\psi_\k\in\Pi_m^d$, we define
\be\label{poly_to_gauss}
\mathfrak{G}_d(P):=\sum_{|\k|_{1,d}<m^2}b_\k\mathfrak{G}_{\k,m,d}.
\ee

\begin{cor}\label{cor:poly_gaussian}
Let $m\ge 1$, $P\in\Pi_m^d$. Then
\be\label{poly_gaussian_approx}
\left\|P-\mathfrak{G}_d(P)\right\|_{\RR^d}\le c_1m^c3^{-m^2/2}\|P\|_{\RR^d}.
\ee
We note that the centers and the number of neurons in the network
 $\sum_{|\k|_{1,d}<m^2}b_\k\mathfrak{G}_{\k,m,d}$ are independent of $P$. In particular, the  number of neurons is $\O(m^{2d})$. 
 \end{cor}

\subsection{Function approximation}\label{bhag:approx}
In this section, we describe some results on approximation of functions on $\RR^d$.
If $f\in C_0(\RR^d)$, we define its degree of approximation by
\be\label{degapprox}
E_n(\RR^d; f):=\min_{P\in\Pi_n^d}\|f-P\|_{\RR^d}.
\ee
 For $\gamma>0$, the smoothness class $W_\gamma(\RR^d)$ comprises $f\in C_0(\RR^d)$ for which
\be\label{sobolevnormdef}
\|f\|_{W_\gamma(\RR^d)}:=\|f\|_{\RR^d}+\sup_{n\ge 0}2^{n\gamma}E_{2^n}(\RR^d;f) <\infty.
\ee
We need some results from \cite{mhasbk, tenswt}, reformulated in the  form stated in Theorem~\ref{theo:equiv} below. 
To state this theorem, we need some notation first.
First, for $\delta\in (0,1]$, $\x\in\RR^d$, $1\le k\le d$, we write
\be\label{qkdef}
Q_{k,\delta}'(\x):=\min(\delta^{-1},|x_k|).
\ee
 For $t>0$ and integer $j\ge 0$, the forward difference of a function $f:\RR^d\to\RR$ is defined by
$$
\Delta_{k,t}^j(f)(\x) := \sum_{\ell=0}^{j}(-1)^{j-\ell}\binom{j}{\ell}f(x_1, \cdots, x_{k-1}, x_k+\ell t, x_{k+1},\cdots,x_d).
$$
and for integers $r\ge 1$
\be\label{omega_def}
\omega_r(f,\delta):= \sum_{k=1}^d\sum_{j=0}^r \delta^{r-j}\sup_{|t|\le\delta}\|(Q_{k,\delta}')^{r-j}\Delta_{k,t}^j (f)\|_{\RR^d}.
\ee

\begin{rem}\label{rem:scaling}
{\rm
If $\lambda>0$, $f_\lambda(\x)=f(\x/\lambda)$, then $\Delta_{k,t}^j(f_\lambda)(\x)=\Delta_{k,t/\lambda}^j(f)(\x/\lambda)$, and $Q_{k,\delta}'(\x)=\lambda Q_{k,\lambda\delta}'(\x/\lambda)$. 
Using the fact that $\delta\mapsto \delta Q_{k,\delta}(\x)$ is non-decreasing for every $\x$ , it is not difficult to deduce that
\bea\label{modulusscaling}
\omega_r(f_\lambda,\delta)&=&\sum_{k=1}^d\sum_{j=0}^r \delta^{r-j}\sup_{|t|\le\delta}\|(Q_{k,\delta}')^{r-j}\Delta_{k,t}^j (f_\lambda)\|_{\RR^d}=\sum_{k=1}^d\sum_{j=0}^r (\lambda\delta)^{r-j}\sup_{|u|\le\delta/\lambda}\|(Q_{k,\lambda\delta}')^{r-j}\Delta_{k,u}^j (f)\|_{\RR^d}\nonumber\\
&\le& \omega_r(f,\max(\lambda, 1/\lambda)\delta).
\eea
}
\end{rem}

\begin{theorem}\label{theo:equiv}
Let $f\in C_0(\RR^d)$, $r\ge 1$, $0<\gamma<r$. Then \\
{\rm (a)} For $n\ge 1$,
\be\label{directest}
E_n(\RR^d;f)\le c\omega_r(f,1/n).
\ee
{\rm (b)} The function
$f\in W_\gamma(\RR^d)$ if and only if $\omega_r(f,\delta)=\O(\delta^\gamma)$ for $0<\delta\le 1$.
In fact,
\be\label{equivineq}
\|f\|_{W_\gamma(\RR^d)}\sim \|f\|_{\RR^d} + \sup_{0<\delta\le 1}\delta^{-\gamma} \omega_r(f,\delta).
\ee
\end{theorem}

\begin{proof}\ 
The theorem is already contained in the results in \cite{tenswt}, but we need to reconcile notation and explain why. 
In \cite[Formulas~(42),(43)]{tenswt} we have defined a univariate $K$-functional and a pre-modulus of smoothness  for  $g(\x)=\exp(|\x|_{2,d}^2/2)f(\x)$ applied to the $k$-th component of $\x$,  $k=1,\cdots,d$. 
The $K$-functional obtained in this way is denoted in \cite[Formula~(21)]{tenswt} by $K _{r,k}$. 
Likewise, the quantity denoted by $\omega_r$ in \cite{tenswt} is the $k$-th summand of the right hand side of \eref{omega_def}. 
Our definition of $Q_{k,\delta}'$ is slightly different from that in \cite{tenswt} (where it is defined to be $\min(\delta^{-1},(1+x_k^2)^{1/2})$). However, our $Q_{k,\delta}'$ as defined in \eqref{qkdef} satisfies $Q_{k,\delta}'\sim\min(\delta^{-1},(1+x_k^2)^{1/2})$, 
Therefore, \cite[Theorem~5.1, Proposition~4.5]{tenswt} lead to the statement of this theorem.
\end{proof}

\begin{rem}\label{rem:equiv_theo}
{\rm
If $\gamma=r+\beta$, where $r\ge 0$ is an integer and $0< \beta\le 1$,  $f\in C_0^r(\RR^d)$ and satisfies
\be\label{altlipcond}
\sup_{|\u|_{2,d}\le \delta}\|f^{(r)}(\circ+\u)-f^{(r)}\|_{\RR^d} + \delta \left\|\min(\delta^{-1}, |\circ|_{2,d})f^{(r)}\right\|_{\RR^d} \le c(f)\delta^\beta,
\ee
for every derivative $f^{(r)}$ of order $r$, then $\omega_r(f,\delta)=\O(\delta^\gamma)$ for $0<\delta\le 1$, and $f\in W_\gamma(\RR^d)$. 
 If $f\in C_0^r(\RR^d)$ is compactly supported, and every derivative $f^{(r)}$ of order $r$ satisfies  
 $$
 \sup_{|\u|_{2,d}\le \delta}\|f^{(r)}(\circ+\u)-f^{(r)}\|_{\RR^d}\le c(f)\delta^\beta,
 $$
then $f\in W_\gamma(\RR^d)$.
In particular, if $f$ is compactly supported and satisfies a Lipschitz condition, then $f\in W_1(\RR^d)$, and therefore, also $f\in W_\gamma(\RR^d)$ for every $\gamma\in (0,1)$.
  \qed
}
\end{rem}

We define
\be\label{sigmadef}
\sigma_n(\RR^d;f)(\x):=\int_{\RR^q}\Phi_{n,d}(\x,\y)f(\y)d\y, \qquad f\in C_0(\RR^d),\ n>0, \ \x\in\RR^d.
\ee
The following proposition is routine to prove using Proposition~\ref{prop:kernloc}:
\begin{prop}\label{prop:goodapprox}
{\rm (a)} If $n>0$ and $P\in \Pi_{n/\sqrt{2}}^d$, then $\sigma_n(\RR^d;P)=P$.\\
{\rm (b)} If $f\in C_0(\RR^d)$, $n>0$, then
\be\label{sigmaopbd}
\|\sigma_n(\RR^d;f)\|_{\RR^d}\le c\|f\|_{\RR^d}, \quad E_n(\RR^d;f)\le \|f-\sigma_n(\RR^d;f)\|_{\RR^d}\le cE_{n/\sqrt{2}}(\RR^d;f).
\ee
\end{prop}

\bhag{Approximation on  affine spaces}\label{bhag:affine}
In the sequel, we fix integers $Q\ge q\ge 1$.

Let $\mathbb{Y}$ be a $q$-dimensional affine subspace of $\RR^Q$, passing through a point $\x_0\in\RR^Q$.
Then there exists a rotation operator $\mathcal{R}$ on $\RR^Q$ depending only on $\mathbb{Y}$ such that any point $\x\in\mathbb{Y}$ can be expressed in the form (with $\bs 0_{Q-q}=(0,\cdots,0)\in\RR^{Q-q}$)
\be\label{ptonaffinespace}
\x=:\x_0+\mathcal{R}(\u, \bs 0_{Q-q}) \qquad \u:=\u(\x):=(u_1(\x),\cdots,u_q(\x)).
\ee
With an abuse of notation, we will write this as $\x=\x_0+\mathcal{R}(\u)$.
In this section only, the function $F :\RR^q\to\RR$ is defined by 
\be\label{XXfndef}
F(\u):=f\left(\x_0+\mathcal{R}(\u)\right),
\ee
we define
\be\label{planedegapproxdef}
E_n(\mathbb{Y};f):=E_n(\RR^q;F).
\ee
Similarly, if $\gamma>0$, then $f\in W_\gamma(\mathbb{Y})$ if $F\in W_\gamma(\RR^q)$; i.e.,  $f\in W_\gamma(\mathbb{Y})$ if $f\in C_0(\YY)$ and
\be\label{affinesobnorm}
\|f\|_{W_\gamma(\mathbb{Y})}:=\|F\|_{W_{\gamma}(\RR^q)} <\infty.
\ee
In terms of the points $\x=\x_0+\mathcal{R}(\u, \bs 0_{Q-q})\in\mathbb{Y}$, the class of approximants of functions on $\mathbb{Y}$ have the form $\x\mapsto P(\x)\exp(-|\x-\x_0|^2/2)$, where $P\in\mathbb{P}_{n^2}^Q$. 
If we are interested only in approximation on $\mathbb{Y}$, we may decide to use some standard point, such as the best approximation to $\bs 0\in\RR^Q$ from $\mathbb{Y}$. 
This section is meant to be preparatory to  Section~\ref{bhag:manifoldproofs} where the results in this section will be used with $\mathbb{Y}$ replaced by the tangent space $\mathbb{T}_{\x_0}(\XX)$ to a manifold $\XX$. 
With this goal in mind, our definition is more natural.
We note that if $f$ is supported on a compact neighborhood of $\x_0$, then $F$ is supported on a compact neighborhood of $\bs 0\in\RR^q$. 
Therefore, for such functions, we may use Theorem~\ref{theo:equiv} (and Remark~\ref{rem:equiv_theo}) with $F$ and get the estimates where the constants do not depend upon $\x_0$, although the space of approximants does.

Our goal in this section is to study the analogue of Proposition~\ref{prop:goodapprox} in the context of approximation on $\mathbb{Y}$.


We denote the volume measure of $\mathbb{Y}$ by $\mu_\mathbb{Y}$, and for $f\in C_0(\mathbb{Y})$, $\lambda>0$, $\x=\x_0+\mathcal{R}(\u)$,
\be\label{planekerndef}
\sigma_{n,\lambda}(\mathbb{Y};f)(\x):=\sigma_{n,\lambda}(\x_0,\mathbb{Y};f)(\x):=\lambda^q\int_\mathbb{Y} \Phi_{n,q,Q}(\lambda(\x-\x_0),\lambda(\y-\x_0))f(\y)d\mu_\mathbb{Y}(\y).
\ee
\begin{theorem}
\label{theo:plane}
Let $Q\ge q\ge 1$ be integers,  $\mathbb{Y}$ be a $q$-dimensional affine subspace of $\RR^Q$, passing through $\x_0\in\RR^Q$, $f\in C_0(\mathbb{Y})$, $\lambda>0$. Then
\be\label{planeest}
\left\|\sigma_{n,\lambda}(\mathbb{Y};f)-f\right\|_\mathbb{Y} \le cE_{n/\sqrt{2}}(\mathbb{Y};f(\x_0+\mathcal{R}((\circ-\x_0)/\lambda)).
\ee
In particular, if $\gamma>0$, $f\in W_\gamma(\mathbb{Y})$, $\lambda\ge 1$, then
\be\label{planeest_scaled}
\left\|\sigma_{n,\lambda}(\mathbb{Y};f)-f\right\|_\mathbb{Y}\le c\|f\|_{W_\gamma(\mathbb{Y})}(\lambda/n)^\gamma.
\ee
Here, all the constants are independent of $\lambda$.
\end{theorem}
\begin{proof}\ 
Since the kernel $\Phi_{n,q,Q}$ is  invariant under rotations, 
it is easy to verify that for $\x=\x_0+\mathcal{R}\u\in\mathbb{Y}$,
$$
\sigma_{n,\lambda}(\mathbb{Y};f)(\x)=\int_{\RR^q}\Phi_{n,q}(\lambda\u, \v)F(\v/\lambda)d\v.
$$
Hence, \eqref{planeest} follows from Proposition~\ref{prop:goodapprox}.
The estimate \eqref{planeest_scaled} follows from Remark~\ref{rem:scaling}.
\end{proof}

\bhag{Proofs of the theorems in Section~\ref{bhag:manifoldapprox}}\label{bhag:manifoldproofs}
For any $\x\in\XX$, we need to consider in this section three kinds of balls, defined in \eqref{balldef}: 
$$
B_Q(\x,r):=\{\y\in \RR^Q : |\x-\y|_{2,Q} \le r\},\
B_\TT(\x,r):= \TT_\x(\XX)\cap B_Q(\x,r),\ 
\BB(\x,r):=\{\y\in \XX: \rho(\x,\y)\le r\}.
$$
Clearly, if $r\le \iota^*$, then $\BB(\x,r)=\E_\x(B_\TT(\x,r))$.

The following proposition is not difficult to prove using definitions and Taylor expansions (cf. \cite{belkinfound}).
In this section, we will simplify the notation to write $d\u$ in place of $d\mu_{\TT_\x(\XX)}(\u)$.
\begin{prop}\label{prop:expmap}
There exists a constant $C^*>0$ depending only on $\XX$ such that each of the following statements holds for every $\x\in\XX$.\\
\noindent{\rm (a)} We have
\be\label{expdist}
\left||\x-\E_\x(\u)|_{2,Q}-\rho(\x,\E_x(\u))\right|=\left||\x-\E_\x(\u)|_{2,Q}-|\x-\u|_{2,Q}\right|\le C^*\rho(\x,\E_x(\u))^3, \qquad \E_\x(\u)\in \BB(\x,\iota^*).
\ee
{\rm (b)} If $\delta\le \iota^*$ then 
\be\label{expdist2}
|\E_\x(\u)-\u|_{2,Q}\le C^*\delta^2, \qquad \E_\x(\u)\in \BB(\x,\delta),
\ee
{\rm (c)} If $\delta\le \iota^*$ then 
\be\label{expmeasure}
\int_{\BB(\x,\delta)}|d\mu^*(\E_\x(\u))-d\u|\le C^*\delta^{q+2}.
\ee
\end{prop}
\begin{proof}\ 
In this proof only, let $\mathbf{r}$ be any geodesic passing through $\x$, parametrized by the arclength $s$ from $\x$, and $g$ be the metric tensor of $\XX$. 
Then, using the fact that $|\mathbf{r}'(s)|_{2,Q}=1$, and $\mathbf{r}'(s)\cdot\mathbf{r}''(s)=0$, it is easy to deduce using Taylor expansions that for $|s|\le \iota^*$,
$$
||\mathbf{r}(s)-\x|_{2,Q}^2-s^2|\le cs^4; \quad \mbox{i.e., } 1-\frac{|\mathbf{r}(s)-\x|_{2,Q}^2}{s^2}\le cs^2.
$$
Since $1-|\mathbf{r}(s)-\x|_{2,Q}/s \le 1-|\mathbf{r}(s)-\x|_{2,Q}^2/s^2$, this proves \eqref{expdist}.
The estimate \eqref{expdist2} follows from the fact that $\mathbf{r}(s)=\mathcal{E}_\x(\x+s\mathbf{r}'(0))$ and a simple estimate using Taylor theorem.
The estimate \eqref{expmeasure} follows from the well known fact that in exponential coordinates $\sqrt{\det(g)}=1+\O(\delta^2)$ in $\BB(\x,\delta)$ if $\delta\le\iota^*$.
\end{proof}

\begin{cor}\label{cor:ballinclusion}
There exists $C_1^*>0$ depending only on $\XX$ such that for every $\x,\y\in\XX$,
\be\label{distequiv}
|\x-\y|_{2,Q}\le \rho(\x,\y)\le C_1^*|\x-\y|_{2,Q}.
\ee
In particular, for $r>0$,
\be\label{ballinclusion}
\BB(\x,r)\subseteq \BB_Q(\x,r)\subseteq \BB(\x,C_1^*r),
\ee
and \eqref{ballmeasure} is equivalent to
\be\label{ballmeasure2}
\sup_{\x\in\XX, r>0}\frac{\mu^*(\BB_Q(\x,r))}{r^q} \le c.
\ee
\end{cor}

\begin{proof}\ 
In this proof only, let $a=\min((2C^*)^{-1/2},\iota^*/2)$.
Then for $\rho(\x,\y)\le a$, \eqref{expdist} shows that
$$
0\le 1-\frac{|\x-\y|_{2,Q}}{\rho(\x,\y)}\le C^*\rho(\x,\y)^2 \le 1/2.
$$
Therefore,
\be\label{pf8eqn1}
|\x-\y|_{2,Q}\le \rho(\x,\y)\le 2|\x-\y|_{2,Q}, \qquad \mbox{if $\rho(\x,\y)\le a$}.
\ee
In this proof only, let $A=\{(\x,\y)\in\XX\times\XX : \rho(\x,\y)\ge a\}$. 
Then $A$ is a compact set and the function $(\x,\y)\mapsto |\x-\y|_{2,Q}/\rho(\x,\y)$, being continuous on $A$, attains its (necessarily positive) minimum. Thus, there exists $c$ such that 
$$
|\x-\y|_{2,Q}\le \rho(\x,\y)\le c|\x-\y|_{2,Q}, \qquad \mbox{if $\rho(\x,\y)\ge a$}.
$$
Together with \eqref{pf8eqn1}, this leads to \eqref{distequiv}, and hence to \eqref{ballinclusion}.
\end{proof}

To motivate the construction of the operator for approximation, our idea is to transfer the target function locally at each point to the tangent space at that point. 
Therefore, we use the operator defined as in Section~\ref{bhag:affine}.
In the present situation, at any point $\x$ at which the approximation is desired, the affine space passes through the point $\x$ itself, which plays the dual role of $\x_0$ in 
Section~\ref{bhag:affine}. 
While there is only one parameter $t$ in Theorem~\ref{theo:belkin_niyogi}, our construction allows us to have two parameters to control localization: the parameter $n$ controlling the degree of the polynomials involved and an additional parameter to control scaling.
Recalling that $\Phi_{n,q,Q}(\x,\y)=\Phi_{n,q,Q}(-\x,-\y)$ we can define our operator as a convolution as follows.
\be\label{manifold_op_def}
\sigma_{n,\lambda}(\XX;f)(\x):=\lambda^q\int_\XX \Phi_{n,q,Q}(\bs 0,\lambda(\y-\x))f(\y)d\mu^*(\y)=\lambda^q\int_\XX \widetilde{\Phi}_{n,q,Q}(\lambda|\x-\y|_{2,Q})f(\y)d\mu^*(\y).
\ee
Our first theorem is the analogue of Theorem~\ref{theo:plane} when $\XX$ is a manifold instead of an affine space. 

\begin{theorem}\label{theo:main}
Let $\gamma>0$, $f\in W_\gamma(\XX)$, $0<\alpha\le 1$, $\alpha<4/(\gamma+2)$.
Then for $n\ge 1$, $\lambda=n^{1-\alpha}$,
\be\label{fundaapprox}
\|f-\sigma_{n,\lambda}(\XX;f)\|_\XX\le cn^{-\alpha\gamma}\|f\|_{W_\gamma(\XX)}.
\ee 
\end{theorem}  

It is convenient to summarize some details of the proof of this theorem in the form of the following lemma.

\begin{lemma}\label{lemma:manifoldkern}
Let  $\x\in\XX$, $g\in C(\XX)$ be supported on $\BB(\x,\iota^*/8)$, $G(\u)=g(\E_\x(\u))$,  $\gamma>0$, $0<\alpha\le 1$, $\alpha<4/(\gamma+2)$.  Then for $n\ge 1$, $\lambda=n^{1-\alpha}$, 
\be\label{manifold_to_euclid}
\left|\lambda^q\int_\XX \widetilde{\Phi}_{n,q}(\lambda|\x-\y|_{2,Q})g(\y)d\mu^*(\y)-\lambda^q\int_{\TT_\x(\XX)}\widetilde{\Phi}_{n,q}(\lambda|\x-\u|_{2,Q})G(\u)d\u\right|\le c n^{-\alpha\gamma}\|g\|_\XX,
\ee
where $G$ is extended outside $\BB_\TT(\x,\iota^*/8)$ as a zero function.
\end{lemma}

\begin{proof}\ 

Without loss of generality, we assume that $\|g\|_\XX=1$.
First, we summarize our choices of various parameters.

In this proof only, let 
$$
\delta=n^{-((2-\alpha)(q+1)+\alpha\gamma)/(q+3)},
$$
so that  for sufficiently large $n$,
\be\label{pf2eqn11}
\delta<\min (1,\iota^*/6), \quad n^{q+1}\lambda^{q+1}\delta^{q+3}=n^{-\alpha\gamma}, \quad n\lambda\delta =n^{(4-\alpha\gamma-2\alpha)/(q+3)}\uparrow \infty.
\ee
We choose
\be\label{pf2eqn8}
S\ge\frac{(q(2-\alpha)+\alpha\gamma+1)(q+3)}{4-\alpha\gamma-2\alpha}, \quad (\Rightarrow n^q\lambda^q(n\lambda\delta)^{-S}\le n^{-\alpha\gamma-1}).
\ee
We now assume further that $n$ is large enough so that with $C^*$ as in Proposition~\ref{prop:expmap}, $C^*\delta^2 \le \delta/2$. 

Next,  we summarize the implications of our choices on the distances on the manifold, tangent space, and the ambient space.

If $\y\in \BB(\x,\iota^*/8)\cap B_Q(\x,\delta)$, $\u\in\TT_\x(\XX)$, $\y=\E_\x(\u)$, then  \eref{expdist2} shows that
\be\label{pf2eqn7}
|\x-\u|_{2,Q}\le |\x-\y|_{2,Q}+|\E_\x(\u)-\u|_{2,Q}\le \delta +C^*\delta^2 \le (3/2)\delta, \qquad \rho(\x,\y)\le 3\delta <\iota^*/2.
\ee
Thus,
\be\label{pf2eqn10}
E_\delta:=\E_\x^{-1}(\BB(\x,\iota^*/8)\cap B_Q(\x,\delta)) \subseteq B_\TT(\x,3\delta/2).
\ee
If $\u\in B_\TT(\x,\iota^*/8)$ then $\E_\x(\u)$ is well defined. 
If $\u\in B_\TT(\x,\iota^*/8)\setminus E_\delta$, then 
 \eref{expdist2}, \eref{expdist} show that 
\be\label{pf2eqn12}
 |\x-\u|_{2,Q}\ge |\x-\E_\x(\u)|_{2,Q}-|\E_\x(\u)-\u|_{2,Q}\ge \delta -C^*\delta^2\ge \delta/2.
\ee

With this preparation, we are now ready to start with the main estimates.
Since $g$ is supported on $\BB(\x,\iota^*/8)$, we find that (cf. \eref{pf2eqn8}, \eqref{tildephiloc})
\bea\label{pf2eqn2}
\int_{\XX\setminus B_Q(\x,\delta)}|\widetilde{\Phi}_{n,q}(\lambda|\x-\y|_{2,Q})g(\y)|d\mu^*(\y)&=&\int_{\BB(\x,\iota^*/8)\setminus B_Q(\x,\delta)}|\widetilde{\Phi}_{n,q}(\lambda|\x-\y|_{2,Q})g(\y)|d\mu^*(\y)\nonumber\\
&\le& cn^q (n\lambda\delta)^{-S}\le cn^{-\alpha\gamma-1}\lambda^{-q}.
\eea

Using \eqref{tildephibern} and \eqref{expdist}, we deduce that for $\y=\E_\x(\u)\in \BB(\x,\iota^*/8)\cap B_Q(\x,\delta)$,
\be\label{pf2eqn3}
|\widetilde{\Phi}_{n,q}(\lambda|\x-\E_\x(\u)|_{2,Q})-\widetilde{\Phi}_{n,q}(\lambda|\x-\u|_{2,Q})| \le cn^{q+1}\lambda\left||\x-\E_\x(\u)|_{2,Q}-|\x-\u|_{2,Q}\right| \le n^{q+1}\lambda\delta^3.
\ee
The estimates \eref{pf2eqn7} and \eref{ballmeasure} lead further to
\be\label{pf2eqn4}
\left|\int_{\BB(\x,\iota^*/8)\cap B_Q(\x,\delta)}d\mu^*(\y)- \int_{E_\delta}d\u\right| \le\left|\int_{E_\delta}|d\mu^*(\E_\x(\u))-d\u|\right| \le c\delta^{q+2}.
\ee
In view of \eref{pf2eqn11}, \eref{pf2eqn3} and \eref{pf2eqn4}, we deduce that
\bea\label{pf2eqn5}
\lefteqn{
\left| \int_{\BB(\x,\iota^*/8)\cap B_Q(\x,\delta)}\widetilde{\Phi}_{n,q}(\lambda|\x-\y|_{2,Q})  g(\y)d\mu^*(\y)- \int_{E_\delta}\widetilde{\Phi}_{n,q}(\lambda|\x-\u|_{2,Q})G(\u)d\u\right|}\nonumber\\
&=& \left| \int_{E_\delta}\widetilde{\Phi}_{n,q}(\lambda|\x-\E_\x(\u)|_{2,Q})G(\u)d\mu^*(\E_\x(\u))- \int_{E_\delta}\widetilde{\Phi}_{n,q}(\lambda|\x-\u|_{2,Q})G(\u)d\u\right|\nonumber\\
&\le& \left|\int_{E_\delta}\left(\widetilde{\Phi}_{n,q}(\lambda|\x-\E_\x(\u)|_{2,Q})-\widetilde{\Phi}_{n,q}(\lambda|\x-\u|_{2,Q})\right)G(\u)d\u\right| +cn^q\delta^{q+2}\nonumber\\
&\le& cn^{q+1}\lambda\delta^{q+3}=cn^{-\alpha\gamma}\lambda^{-q}.
\eea
The localization estimate \eqref{tildephiloc} shows (cf. \eref{pf2eqn8}) that
\be\label{pf2eqn9}
\left|\int_{\TT_\x(\XX) \setminus B_\TT(\x,\iota^*/8)}\widetilde{\Phi}_{n,q}(\lambda|\x-\u|_{2,Q})G(\u)d\u\right|\le cn^q(n\lambda)^{-S}\le cn^{-\alpha\gamma-1}\lambda^{-q}.
\ee

Invoking the localization estimate \eref{tildephiloc} and \eref{pf2eqn11}, \eref{pf2eqn12} again, we deduce  that
\be\label{pf2eqn6}
\left|\int_{B_\TT(\x,\iota^*/8)\setminus E_\delta} \widetilde{\Phi}_{n,q}(\lambda|\x-\u|_{2,Q})G(\u)d\u\right| \le cn^q (n\lambda\delta)^{-S}\le cn^{-\alpha\gamma-1}\lambda^{-q}.
\ee
The estimates \eref{pf2eqn2}, \eref{pf2eqn5}, \eref{pf2eqn9} and \eref{pf2eqn6} lead to \eref{manifold_to_euclid}.
\end{proof}\\
 
We are now in a position to prove Theorem~\ref{theo:main}.\\

\noindent\textsc{Proof of Theorem~\ref{theo:main}.}\\
Let $\x\in\XX$.  Let $\phi\in C^\infty(\XX)$ be chosen so that $\phi(\y)=1$ if $\y\in\BB(\x,\iota^*/16)$, $\phi(\y)=0$ if $\y\in\XX\setminus \BB(\x,\iota^*/8)$, and $0\le \phi(\y) \le 1$ for $\y\in\XX$.
Then the function $f\phi$ is supported on $\BB(\x,\iota^*/8)$, and hence, the function $F :\TT_\x(\XX)\to \RR$ defined by
$
F(\u):=f(\E_\x(\u))\phi(\E_\x(\u))
$
is in $W_\gamma(\TT_\x(\XX))$.
Clearly,
$$
\|F\|_{\TT_\x(\XX)}\le \|f\|_\XX, \qquad \|F\|_{W_\gamma(\TT_\x(\XX))}\le \|f\|_{W_\gamma(\XX)}.
$$

We choose $S> q+(\alpha\gamma+1)/(2-\alpha)$, and write $a=\iota^*/(16C_1^*)$, where $C_1^*$ is the constant defined in Corollary~\ref{cor:ballinclusion}. 
Then, the inclusion \eqref{ballinclusion} and the localization property \eqref{tildephiloc} show that
\be\label{pf3eqn1}
\begin{aligned}
\left|\int_\XX\widetilde{\Phi}_{n,q}(\lambda|\x-\y|_{2,Q})\right.&\left.(1-\phi(\y))f(\y)d\mu^*(\y)\right|= \left|\int_{\XX\setminus \BB(\x,\iota^*/16)}\widetilde{\Phi}_{n,q}(\lambda|\x-\y|_{2,Q})(1-\phi(\y))f(\y)d\mu^*(\y)\right|\\
&\le \int_{\XX\setminus \BB_Q(\x,a)}\left|\widetilde{\Phi}_{n,q}(\lambda|\x-\y|_{2,Q})(1-\phi(\y))f(\y)\right|d\mu^*(\y)\\
&\le cn^{q-S}n^{-(1-\alpha)S}\|f\|_\XX\le cn^{-\alpha\gamma-1}\lambda^{-q}\|f\|_\XX.
\end{aligned}
\ee
In view of Lemma~\ref{lemma:manifoldkern},
\be\label{pf3eqn2}
\left|\int_\XX \widetilde{\Phi}_{n,q}(\lambda|\x-\y|_{2,Q})\phi(\y)f(\y)d\mu^*(\y)-\int_{\TT_\x(\XX)}\widetilde{\Phi}_{n,q}(\lambda|\x-\u|_{2,Q})F(\u)d\u\right|\le cn^{-\alpha\gamma}\lambda^{-q}\|f\|_\XX,
\ee
so that
\be\label{pf3eqn3}
\left|\int_\XX \widetilde{\Phi}_{n,q}(\lambda|\x-\y|_{2,Q})f(\y)d\mu^*(\y)-\int_{\TT_\x(\XX)}\widetilde{\Phi}_{n,q}(\lambda|\x-\u|_{2,Q})F(\u)d\u\right|\le cn^{-\alpha\gamma}\lambda^{-q}\|f\|_\XX.
\ee
Since $F(\x)=f(\x)$, \eqref{planeest_scaled} in Theorem~\ref{theo:plane} now shows that 
\be\label{pf3eqn4}
\left|\lambda^q\int_\XX \widetilde{\Phi}_{n,q}(\lambda|\x-\y|_{2,Q})f(\y)d\mu^*(\y)-f(\x)\right| \le c(n/\lambda)^{-\gamma}\|F\|_{W_\gamma(\TT_\x(\XX))}\le cn^{-\alpha\gamma}\|f\|_{W_\gamma(\XX)}.
\ee
This proves \eref{fundaapprox}.
\qed

Our next objective in this section is to obtain the following discretization of Theorem~\ref{theo:main} based on noise-corrupted random samples of $f$ as in Theorem~\ref{theo:manifoldprob}.

The proof of Theorem~\ref{theo:manifoldprob} is included in that of the following theorem, together with Theorem~\ref{theo:main} applied with $ff_0$ in place of $f$.

\begin{theorem}\label{theo:prob_reconstr}
We assume the set up as in Theorem~\ref{theo:manifoldprob}.
 Then   for every $n\ge 1$ and $M\ge n^{q(2-\alpha)+2\alpha\gamma}\log(n/\delta)$ we have with $\lambda=n^{1-\alpha}$,
\be\label{masterprobest}
\mathsf{Prob}_{\tau}\left(\left\|\widehat{F}_{n,\alpha} (Y;\circ)-\sigma_{n,\lambda}(\XX;ff_0)\right\|_{\RR^Q}\ge  c\sqrt{\|f_0\|_{\XX}}\|\mathcal{F}\|_{\XX\times \Omega} n^{-\alpha\gamma}\right)\le \delta.
\ee
\end{theorem}

The proof of Theorem~\ref{theo:prob_reconstr} requires some preparation.
We start with the following concentration inequality \cite[Section~2.7]{boucheron2013concentration}.
\begin{prop}\label{prop:concentration}
 (\textbf{Bernstein concentration inequality}) Let $Z_1,\cdots, Z_M$ be independent real valued random variables such that for each $j=1,\cdots,M$, $|Z_j|\le R$, and $\mathbb{E}(Z_j^2)\le V$. Then for any $t>0$,
\be\label{bernstein_concentration}
\mathsf{Prob}\left( \left|\frac{1}{M}\sum_{j=1}^M (Z_j-\mathbb{E}(Z_j))\right| \ge t\right) \le 2\exp\left(-\frac{Mt^2}{2(V+Rt/3)}\right).
\ee
\end{prop}

In order to apply Proposition~\ref{prop:concentration}, we need to estimate the second moment of $\mathcal{F}(\y,\epsilon)\widetilde{\Phi}_{n,q,Q}(\lambda|\x-\y|_{2,Q})=\mathcal{F}(\y,\epsilon)\Phi_{n,q,Q}(\bs 0,\lambda(\x-\y))$ for every $\x\in\RR^Q$. 
This is done in the following lemma.

\begin{lemma}\label{lemma:variance}
We have
\be\label{variance}
\lambda^{2q}\sup_{\x\in\RR^Q}\int_{\XX\times\Omega} |\mathcal{F}(\y,\epsilon)\Phi_{n,q,Q}(\bs 0, \lambda(\x-\y))|^2d\tau(\y,\epsilon) \le c(n\lambda)^q\|\mathcal{F}\|_{\XX\times \Omega}^2\|f_0\|_\XX.
\ee
\end{lemma}
\begin{proof}\ 
Let $\x\in\RR^Q$. We need only to estimate 
\be\label{pf6eqn1}
\int_{\XX\times\Omega} \left|\mathcal{F}(\y,\epsilon)\Phi_{n,q,Q}(\bs 0, \lambda(\x-\y))\right|^2d\tau(\y,\epsilon)\le \|\mathcal{F}\|_{\XX\times \Omega}^2\|f_0\|_\XX \int_\XX \Phi_{n,q,Q}(\bs 0, \lambda(\x-\y))^2d\mu^*(\y).
\ee
Using Proposition~\ref{prop:Q_to_q_reduce} and \eref{ballmeasure2}, and keeping in mind that $\lambda\ge 1$, we deduce that
\begin{eqnarray*}
\begin{aligned}
\int_\XX \Phi_{n,q,Q}(\bs 0, \lambda(\x-\y))^2d\mu^*(\y)&=\int_{\XX\cap \BB_Q(\x,1/( n\lambda))}\Phi_{n,q,Q}(\bs 0, \lambda(\x-\y))^2d\mu^*(\y)\\
&+\sum_{k=0}^\infty \int_{\XX\cap(\BB_Q(\x,2^{k+1}/(n\lambda))\setminus \BB_Q(\x,2^k/(n\lambda)))}\Phi_{n,q,Q}(\bs 0, \lambda(\x-\y))^2d\mu^*(\y)\\
&\le cn^{2q}\left\{\mu^*(\BB_Q(\x,1/(n\lambda))) +\sum_{k=0}^\infty 2^{-2kS}\mu^*(\BB_Q(\x,2^{k+1}/(n\lambda))\setminus \BB_Q(\x,2^k/(n\lambda)))\right\}\\
&\le cn^q\lambda^{-q}\left\{1+\sum_{k=0}^\infty 2^{-k(2S-q)}\right\}\le cn^q\lambda^{-q}.
\end{aligned}
\end{eqnarray*}
\end{proof}

The proof of Theorem~\ref{theo:prob_reconstr} requires an estimation of a quantity of the form
$$
\sup_{\y_1,\cdots,\y_M\in \XX}\left\|\frac{\lambda^q}{M}\sum_{j=1}^M  \mathcal{F}(\y_j,\epsilon_j)\Phi_{n,q,Q}(\bs 0, \lambda(\circ-\y_j))-\sigma_{n,\lambda}(\XX;ff_0)\right\|_{\RR^Q}
$$
in terms of the maximum of the function involved at finitely many points.
The following lemma accomplishes this by considering the difference between two measures on $\XX$: one that associates the mass $(1/M)\mathcal{F}(\y_j,\epsilon_j)$ with each $\y_j$, and other given by $f(\y)d\nu^*(\y)=f(\y)f_0(\y)d\mu^*(\y)$. 
We will denote the total variation of a measure $\nu$ by $\tn\nu\tn_{TV}$. 
The total variation of the difference between the two measures mentioned above is clearly $\le 2\|\mathcal{F}\|_{\XX\times\Omega}$. 
\begin{lemma}\label{lemma:infinite_to_finite}
Let $S>Q+2$, $\lambda$ be as in Theorem~\ref{theo:main}. There exists $c^*=c^*(S)>0$ and a finite set $\mathcal{D}^*\subset \RR^Q$ with $|\mathcal{D}^*| \sim n^{c^*}$ such that for any measure $\nu$ on $\XX$,
\be\label{infinite_finite}
\left\|\lambda^q\int_\XX  \Phi_{n,q,Q}(\bs 0, \lambda(\circ-\y))d\nu(\y)\right\|_{\RR^Q} \le \max_{\x\in\mathcal{D}^*}\left|\lambda^q\int_\XX  \Phi_{n,q,Q}(\bs 0, \lambda(\x-\y))d\nu(\y)\right| + cn^{-S} \tn\nu\tn_{TV}.
\ee
\end{lemma} 

\begin{proof}\ 
We assume $n$ to be large enough so that $\XX\subset [-\sqrt{2}n,\sqrt{2}n]^Q$. 
Then Proposition~\ref{prop:Q_q_lockern} (used with $2S$ in place of $S$) shows that 
\be\label{pf7eqn1}
\sup_{\x\in \RR^Q \setminus[-2n,2n]^Q}\left|\int_\XX  \Phi_{n,q,Q}(\bs 0, \lambda(\x-\y))d\nu(\y)\right| \le cn^{Q-2S}\tn\nu\tn_{TV}\le cn^{-S}\tn\nu\tn_{TV}.
\ee
Therefore,
\be\label{pf7eqn2}
\left\|\int_\XX  \Phi_{n,q,Q}(\bs 0, \lambda(\circ-\y))d\nu(\y)\right\|_{\RR^Q}\le \sup_{\x\in [-2n,2n]^Q}\left|\int_\XX  \Phi_{n,q,Q}(\bs 0, \lambda(\x-\y))d\nu(\y)\right|+cn^{-S}\tn\nu\tn_{TV}.
\ee

Next, we observe that for any $\y\in\RR^Q$
$$
\nabla_\x\left(\Phi_{n,q,Q}(\bs 0, \lambda(\x-\y))\right) = \lambda \left(\nabla_\x\Phi_{n,q,Q}(\bs 0, \circ)\right)(\lambda(\x-\y)).
$$
Therefore, for any $\x\in\RR^Q$,
$$
\left|\nabla_\x\left(\int_\XX  \Phi_{n,q,Q}(\bs 0, \lambda(\x-\y))d\nu(\y)\right)\right| \le \lambda\int_\XX \left|\left(\nabla_\x\Phi_{n,q,Q}(\bs 0, \circ))\right)(\lambda(\x-\y))\right|d|\nu|(\y).
$$
Using the Bernstein inequality Proposition~\ref{prop:wtpoly}(c),  we conclude that
$$
\sup_{\x\in\RR^Q}\left|\nabla_\x\left(\int_\XX  \Phi_{n,q,Q}(\bs 0, \lambda(\x-\y))d\nu(\y)\right)\right| \le  cn^{q+1}\lambda \tn\nu\tn_{TV} =c n^{q+2-\alpha}\tn\nu\tn_{TV}.
$$
and hence, for any $\x,\z\in\RR^Q$,
\be\label{pf7eqn3}
\left|\lambda^q\int_\XX  \Phi_{n,q,Q}(\bs 0, \lambda(\x-\y))d\nu(\y)-\lambda^q\int_\XX  \Phi_{n,q,Q}(\bs 0, \lambda(\z-\y))d\nu(\y)\right| \le c n^{(q+1)(2-\alpha)}\tn\nu\tn_{TV}|\x-\z|_{\infty,Q}.
\ee
We now let $\mathcal{D}^*$ be a finite subset of $[-2n,2n]^Q$ such that
\be\label{pf7eqn4}
\max_{\x\in [-2n,2n]^Q}\min_{\z\in \mathcal{D}^*}|\x-\z|_{\infty,Q} \le n^{-(q+1)(2-\alpha)-S},
\ee
and observe that $|\mathcal{D}^*|\sim n^{Q((q+1)(2-\alpha)+S)}$.
The estimate \eref{infinite_finite} is easy to deduce using \eref{pf7eqn2}, \eref{pf7eqn3}, and \eref{pf7eqn4}.
\end{proof}\\

With this preparation, we now prove Theorem~\ref{theo:prob_reconstr}, and hence, Theorem~\ref{theo:manifoldprob}.\\

\noindent\textsc{Proof  of Theorem~\ref{theo:prob_reconstr} (and Theorem~\ref{theo:manifoldprob}).}\\

Let $\x\in\RR^Q$. 
We consider the random variables
\be\label{pf4eqn1}
Z_j(\x)=\lambda^q\mathcal{F}(\y_j,\epsilon_j) \Phi_{n,q,Q}(\bs 0, \lambda(\x-\y_j)).
\ee
It is easy to verify using Fubini's theorem that if $\mathcal{F}$ is integrable with respect to $\tau$ then
for any $\x\in\RR^Q$,
\be\label{iterexp}
\mathbb{E}_\tau(\lambda^q\mathcal{F}(\y,\epsilon)\Phi_n(\bs 0,\lambda(\x-\y))=\sigma_{n,\lambda}(\XX;ff_0)(\x).
\ee
The estimate \eref{hermite_max_norm} implies that $|Z_j|\le c(n\lambda)^q\|\mathcal{F}\|_{\XX\times\Omega}$.
Further, Lemma~\ref{lemma:variance} yields $\mathbb{E}_\tau(Z_j^2)\le c(n\lambda)^q\|\mathcal{F}\|_{\XX\times\Omega}^2\|f_0\|_\XX$.
Therefore, we deduce using Proposition~\ref{prop:concentration} that for any $t\in (0,1)$,
\be\label{pf4eqn3}
\mathsf{Prob}_\tau\left(\left|\frac{1}{M}\sum_{j=1}^M Z_j(\x)-\sigma_{n,\lambda}(\XX;ff_0)(\x)\right|\ge  t\|\mathcal{F}\|_{\XX\times\Omega}\|f_0\|_\XX/2\right) \le 2\exp\left(-c\frac{M\|f_0\|_\XX t^2}{(n\lambda)^q}\right).
\ee
In view of Lemma~\ref{lemma:infinite_to_finite}, we have for $S\ge Q+2+\alpha\gamma$,
\be\label{pf4eqn5}
\mathsf{Prob}_\tau\left(\left\|\frac{1}{M}\sum_{j=1}^M Z_j-\sigma_{n,\lambda}(\XX;ff_0)\right\|_{\RR^Q}\ge t\|\mathcal{F}\|_{\XX\times\Omega}\|f_0\|_\XX +c_2n^{-S}\|\mathcal{F}\|_{\XX\times\Omega}\right) \le c_1n^{c^*}\exp\left(-c\frac{M\|f_0\|_\XX t^2}{(n\lambda)^q}\right).
\ee
We recall that $n\lambda=n^{2-\alpha}$ and choose 
$$
t=c_3\sqrt{\frac{n^{q(2-\alpha)}}{M\|f_0\|_\XX}\log(n/\delta)}
$$
for a suitable constant to make the right hand side of \eref{pf4eqn5} to be $\le \delta$, to obtain
\be\label{pf4eq6}
\mathsf{Prob}_\tau\left(\left\|\frac{1}{M}\sum_{j=1}^M Z_j-\sigma_{n,\lambda}(\XX;ff_0)\right\|_{\RR^Q}\ge c_2\|\mathcal{F}\|_{\XX\times\Omega}\left(\sqrt{\frac{n^{q(2-\alpha)}\|f_0\|_\XX}{M}\log(n/\delta)}+n^{-S}\right)\right) \le \delta.
\ee
We now observe that since $1=\int_\XX f_0d\mu^*$, and $\mu^*(\XX)=1$, $\|f_0\|_\XX \ge 1$. Therefore, choosing $M\ge n^{q(2-\alpha)+2\alpha\gamma}\sqrt{\log (n/\delta)}$, we arrive at \eref{masterprobest}.
\qed\\

Theorem~\ref{theo:manifold_approx_prob} is obtained immediately
from Theorem~\ref{theo:manifoldprob} by setting $f_0\equiv 1$. 
To obtain Theorem~\ref{theo:belkin_niyogi_analogue}, we use Theorem~\ref{theo:manifoldprob} once as stated and again with $\mathcal{F}(Y;\circ)\equiv 1$ to get an approximation to $f_0$.

\vskip-0.5cm

\bhag{Proof of the theorems in Section~\ref{bhag:gaussnet}}\label{bhag:deepproofs}

\noindent\textsc{Proof of Theorem~\ref{theo:shallow_net}}\\
Theorem~\ref{theo:shallow_net} follows easily from Theorem~\ref{theo:prob_reconstr} and Corollary~\ref{cor:poly_gaussian}. \\

\noindent\textsc{Proof of Theorem~\ref{theo:good_propogation}}.

Let $v\in V$, and $u_1,\cdots, u_{d(v)}$ be the children of $v$, and  $\x_1,\cdots,\x_{d(v)}$ be the inputs seen by these in that order. 
Let $\x$ be the corresponding input seen by $v$. 
Then using the Lipschitz condition on $f_v$ and the property \eref{pooling_cond}, we obtain
\be\label{pf5eqn1}
\begin{aligned}
|f_v(\x)&-g_v(\x)|=\left|f_v\left(\pi_v\left((f_{u_1}(\x_{u_1}), \cdots, f_{u_{d(v)}}(\x_{u_{d(v)}}))\right)\right)-g_v\left(\pi_v\left((g_{u_1}(\x_{u_1}), \cdots, g_{u_{d(v)}}(\x_{u_{d(v)}}))\right)\right)\right|\\
&\le \left|f_v\left(\pi_v\left((f_{u_1}(\x_{u_1}), \cdots, f_{u_{d(v)}}(\x_{u_{d(v)}}))\right)\right)-f_v\left(\pi_v\left((g_{u_1}(\x_{u_1}), \cdots, g_{u_{d(v)}}(\x_{u_{d(v)}}))\right)\right)\right|\\
&\qquad +\left|f_v\left(\pi_v\left((g_{u_1}(\x_{u_1}), \cdots, g_{u_{d(v)}}(\x_{u_{d(v)}}))\right)\right)-g_v\left(\pi_v\left((g_{u_1}(\x_{u_1}), \cdots, g_{u_{d(v)}}(\x_{u_{d(v)}}))\right)\right)\right|\\
&\le \|f_v\|_{\Lip(\XX_v)}\rho_v\left(\pi_v(f_{u_1}(\x_{u_1}), \cdots, f_{u_{d(v)}}(\x_{u_{d(v)}})), \pi_v(g_{u_1}(\x_{u_1}), \cdots, g_{u_{d(v)}}(\x_{u_{d(v)}}))\right)\\
&\qquad +\|f_v-g_v\|_{\XX_v}\\
&\le c(v)L\sum_{k=1}^{d(v)}\|f_{u_k}-g_{u_k}\|_{\XX_{u_k}} + \|f_v-g_v\|_{\XX_v} \le c(L,\mathcal{G})\varepsilon.
\end{aligned}
\ee
We now use induction on the level of $v$. 
Thus, if $v^*\in \mathbf{S}$, then the ``shallow network'' estimate implied in Theorem~\ref{theo:shallow_net} is already the one which we want.
Suppose the theorem is proved for the DAGs for which the sink node is at level $\ell\ge 0$.
If $v\in V$, so that its level $\ell\ge 1$, then its children are at level $\ell-1\ge 0$. 
For each of the children, say $u$, we consider the subgraph $\mathcal{G}_u$ of $\mathcal{G}$ comprising only those nodes and edges that culminate in $u$ as the sink node.
We then apply the theorem to each of these subgraphs, and then use \eref{pf5eqn1} to conclude that the statement is true for the subgraph $\mathcal{G}_v$ of $\mathcal{G}$ comprising only those nodes and edges that culminate in $v$ as the sink node.
\qed

\begin{rem}\label{rem:deepnet}
{\rm
Suppose we consider a shallow Gaussian network acting on a $2^s$ dimensional manifold of $\RR^Q$. 
The number of samples required to obtain an accuracy of $n^{-\alpha\gamma}$ predicted by Theorem~\ref{theo:shallow_net} is $\O(n^{2^s(2-\alpha)+2\alpha\gamma}\log n)$. On the other hand, suppose the target function has a compositional structure according to a binary tree, but in addition, for any $v\in V$ with children $u_1, u_2$, the image of  
$(f_{u_1}, f_{u_2})$ forms a curve in $\RR^2$. 
Then the number of samples required to get the same accuracy with the corresponding network is only $\O(n^{2-\alpha+2\alpha\gamma}\log n)$ at each level.
In fact, it seems likely that this is the number of samples in the orignal submanifold of $\RR^Q$ itself, since the  input variables external to the machine are given only at the source nodes.
\qed}
\end{rem}

\vskip -1cm
\bhag{Conclusions}\label{bhag:conclude}

We have given a direct solution to the problem of function approximation if the data is sampled  from a compact, smooth, connected Riemannian manifold, \textbf{without knowing the manifold itself}, except for its dimension. 
Our construction avoids the evaluation of an eigen-decomposition of a matrix or otherwise the need to compute the local charts on the manifold.
Also, the construction avoids any optimization/training in the classical paradigm.

Our construction is universal; i.e., can be used for any target function without any assumption on its prior. 
The approximation error is estimated in the probabilistic sense, and of course, depends upon the smoothness of the target function.
In the case when the data is taken from an affine space, our approximation error does not suffer from any saturation, but can be as small as the smoothness of the target function allows. 
In the general case,  the curvature of the manifold imposes some limitations on how well we can estimate the degree of approximation, but there is no saturation in the sense that if the degree of approximation is better for a function, then it must be ``trivial'' in some sense.

We have extended our results to the case of deep Gaussian networks. However, in this context, they are not completely constructive unless the constituent functions in the DAG defining the deep network are known.

\appendix
\renewcommand{\theequation}{\Roman{section}.\hindu{equation}}
\bhag{Saturation phenomenon}\label{bhag:saturation}
The notation in this section is not the same as that in the rest of the paper, except that $\|\cdot\|_A$ will denote the supremum norm on a set $A$.
A detailed discussion of saturation phenomena in approximation theory can be found in \cite{butzernessel}. 
Intuitively, an \emph{approximation process} on a metric space $A$ is a sequence of operators $U_n :C(A)\to C(A)$ such that $U_n(f)\to f$ uniformly on $A$. 
The process is saturated with the rate $\{\delta_n\}$ if $\|U_n(f)-f\|_A = o(\delta_n)$ as $n\to\infty$ implies that $f$ is \emph{trivial} in some sense (classically $U_n(f)=f$) and there exists a non-trivial function $f$ for which $\|U_n(f)-f\|_A=\O(\delta_n)$. 
We are unable to find in the literature a precise definition that covers the many applications where this phenomenon holds.
As remarked earlier, Theorem~\ref{theo:belkin_niyogi} is one example. 
We give two other examples.
\begin{uda}\label{uda:bernstein}
{\rm
For $f\in C([-1,1])$, the Bernstein polynomial is defined by
$$
B_n(f)(x):=\sum_{k=0}^n \binom{n}{k}f(k/n)x^k(1-x)^{n-k}, \qquad x\in [-1,1],\ n=0,1,\cdots.
$$
The Voronowskaja theorem (\cite[Section~1.6.1]{lorentz2013bernstein}) states that if $f\in C^2([-1,1])$ then uniformly in $x\in [-1,1]$,
$$
\lim_{n\to\infty}\left|n\left(B_n(f)(x)-f(x)\right)-f''(x)\frac{x(1-x)}{2}\right|=0.
$$
Thus, $f\in C^2([-1,1])$, $\|B_n(f)-f\|_{[-1,1]}=\O(1/n)$ and if $\|B_n(f)-f\|_{[-1,1]}=o(1/n)$ then $f''(x)=0$ for $x\in (-1,1)$, so that $f$ is a linear function.
\qed}
\end{uda}

\begin{uda}\label{uda:spline}
{\rm
A function $S : [-1,1]\to\RR$ is called \emph{piecewise constant} with $n$ break-points if there are points $t_0=-1<t_1<\cdots<t_{n+1}=1$ such that $S$ is a constant on each $(t_j, t_{j+1})$, $j=0,\cdots,n$.
We denote the class of all piecewise constants with $n$ break-points by $\mathcal{S}_n$, and define for $f\in C([-1,1])$,
$$
\sigma_n(f):=\inf_{S\in \mathcal{S}_n}\|f-S\|_{[-1,1]}.
$$
We note that the break-points of the approximating function may depend upon the target function $f$. 
It is known (\cite[Chapter~12, Theorem~4.3, Corollary~4.4]{devlorbk}) that if $f$ has a bounded total variation on $[-1,1]$ then $\sigma_n(f)=\O(1/n)$. Moreover, if $f\in C([-1,1])$ and $\sigma_n(f)=o(1/n)$ then $f$ is a constant.
\qed}
\end{uda}

\begin{thenomenclature} 
\begin{multicols}{2}
 \nomgroup{A}

  \item [{$\BB_Q(\x,r)$, $\BB_\TT(\x,r)$, $\BB(\x,r)$}]\begingroup Defined in \eref{balldef}\nomeqref {I.0}
		\nompageref{23}
  \item [{$\Delta_{k,\ell}^j$,$Q_{k,\delta}'$, $\omega_r$}]\begingroup  Section~\ref{bhag:approx}\nomeqref {I.0}
		\nompageref{23}
  \item [{$\iota^*$}]\begingroup Inradius of $\XX$\nomeqref {I.0}
		\nompageref{23}
  \item [{$\lambda$}]\begingroup Scaling factor, typically, $n^{1-\alpha}$\nomeqref {I.0}
		\nompageref{23}
  \item [{$\lambda_{k,m}$, $\lambda_{\k,m}$}]\begingroup Quadrature weights, Section~\ref{bhag:wtedpoly}\nomeqref {I.0}
		\nompageref{23}
  \item [{$\mathbb{G}_{n,q,Q}^*$, $\mathbb{G}_{n,q,Q}$}]\begingroup Special Gaussian network \eref{specialgauss}, \eref{festimatorbis}\nomeqref {I.0}
		\nompageref{23}
  \item [{$\mathbb{P}_n^d$, $\Pi_n^d$}]\begingroup Polynomial spaces Section~\ref{bhag:wtedpoly}\nomeqref {I.0}
		\nompageref{23}
  \item [{$\mathcal{E}_\x$}]\begingroup Exponential map at $\x\in\XX$, $\mathcal{E}_\x :\TT_\x(\XX)\to \XX$\nomeqref {I.0}
		\nompageref{23}
  \item [{$\mathcal{G}$}]\begingroup DAG for deep networks, Section~\ref{bhag:deepnets}\nomeqref {I.0}
		\nompageref{23}
  \item [{$\mathcal{P}_{m,q}$, $\widetilde{\Phi}_{n,q}$}]\begingroup Univariate polynomials defined in \eref{fastproj}, \eref{fastkerndef}\nomeqref {I.0}
		\nompageref{23}
  \item [{$\mathfrak{G}_{\k,m,d}$, $\mathfrak{G}_Q$}]\begingroup Basic Gaussian networks \eref{basic_gaussian}, \eref{poly_to_gauss}\nomeqref {I.0}
		\nompageref{23}
  \item [{$\mathsf{Proj}_{m,d}$, $\mathsf{Proj}_{m,q,Q}$}]\begingroup Projection kernels \eref{projdef}, \eref{genrepkerndef}\nomeqref {I.0}
		\nompageref{23}
  \item [{$\mbox{Lip}(\XX)$}]\begingroup Lipschitz functions on $\XX$\nomeqref {I.0}
		\nompageref{23}
  \item [{$\mu^*$}]\begingroup Volume measure on $\XX$\nomeqref {I.0}
		\nompageref{23}
  \item [{$\Phi_{n,d}$, $\Phi_{n,q,Q}$}]\begingroup Localized kernels \eref{summkerndef}, \eref{modsummkerndef}\nomeqref {I.0}
		\nompageref{23}
  \item [{$\pi_v$}]\begingroup Pooling operation Section~\ref{bhag:deepnets}\nomeqref {I.0}
		\nompageref{23}
  \item [{$\rho$}]\begingroup Metric on $\XX$\nomeqref {I.0}
		\nompageref{23}
  \item [{$\sigma_n$, $\sigma_{n,\lambda}$}]\begingroup Approximation operators \eref{sigmadef}, \eref{planekerndef}, \eref{manifold_op_def}\nomeqref {I.0}
		\nompageref{23}
  \item [{$\tau$}]\begingroup Probability distribution for the data\nomeqref {I.0}
		\nompageref{23}
  \item [{$\TT_\x(\XX)$}]\begingroup Tangent space to $\XX$ at $\x$\nomeqref {I.0}
		\nompageref{23}
  \item [{$\XX$}]\begingroup Manifold\nomeqref {I.0}\nompageref{23}
  \item [{$\YY$}]\begingroup Affine space\nomeqref {I.0}\nompageref{23}
  \item [{$d$}]\begingroup Generic dimension, Section~\ref{bhag:backwtpoly}\nomeqref {I.0}
		\nompageref{23}
  \item [{$d(v)$}]\begingroup Ambient dimension at vertex $v$ Section~\ref{bhag:deepnets}\nomeqref {I.0}
		\nompageref{23}
  \item [{$E_n(A;f)$}]\begingroup Degree of approximation of $f$ on $A$\nomeqref {I.0}
		\nompageref{23}
  \item [{$f$, $\mathcal{F}$, $\widehat{F}_{n,\alpha}$}]\begingroup Target function, observations, and estimator\nomeqref {I.0}
		\nompageref{23}
  \item [{$f_0$}]\begingroup Density of the marginal distribution\nomeqref {I.0}
		\nompageref{23}
  \item [{$H$}]\begingroup Low pass filter Section~\ref{bhag:manifoldback}\nomeqref {I.0}
		\nompageref{23}
  \item [{$h_k$, $\psi_k$, $\psi_\k$}]\begingroup Orthonormalized Hermite polynomial, Hermite function, tensor product Hermite function\nomeqref {I.0}
		\nompageref{23}
  \item [{$n$, $\alpha$}]\begingroup Parameters in approximation\nomeqref {I.0}
		\nompageref{23}
  \item [{$Q$}]\begingroup Dimension of the ambient space\nomeqref {I.0}
		\nompageref{23}
  \item [{$q$}]\begingroup Dimension of affine space  or manifold\nomeqref {I.0}
		\nompageref{23}
  \item [{$q_v$}]\begingroup Dimension in Section~\ref{bhag:deepnets}\nomeqref {I.0}
		\nompageref{23}
  \item [{$S$}]\begingroup Large integer controlling localization\nomeqref {I.0}
		\nompageref{23}
  \item [{$V$, $\mathbf{S}$}]\begingroup Non-source, source vertices Section~\ref{bhag:deepnets}\nomeqref {I.0}
		\nompageref{23}
  \item [{$v_v$}]\begingroup Constituent function at $v$ Section~\ref{bhag:deepnets}\nomeqref {I.0}
		\nompageref{23}
  \item [{$W_\gamma(A)$}]\begingroup Smoothness class on $A$\nomeqref {I.0}
		\nompageref{23}
  \item [{$x_{k,m}$, $\x_{\k,m}$}]\begingroup Quadrature nodes Section~\ref{bhag:wtedpoly}\nomeqref {I.0}
		\nompageref{23}
\end{multicols}
\end{thenomenclature}


\end{document}